\documentclass[12pt]{article}
\usepackage{setspace}
\usepackage[utf8]{inputenc}
\usepackage{amssymb}
\usepackage{amsmath}
\usepackage{amsthm}
\usepackage{hyperref}
\usepackage{a4wide}
\usepackage{breqn}
\usepackage[table,xcdraw]{xcolor}
\usepackage{xcolor}
\usepackage{mathtools}
\usepackage{bm}
\usepackage{enumitem}
\usepackage{mathrsfs}
\usepackage{graphicx}
\usepackage{caption}
\usepackage{url}
\usepackage{euscript}
\usepackage{placeins}
\usepackage{mathdots}
\usepackage{multirow}
\usepackage{arydshln} 
\usepackage[english]{babel}
\usepackage{array}
\usepackage{subcaption}
\usepackage{booktabs}
\usepackage[numbers]{natbib}
\usepackage{cases}

\usepackage{graphics}
\usepackage{graphicx}
\usepackage{algorithm}
\usepackage{euscript}
\usepackage{mathrsfs}
\usepackage{algpseudocode}
\numberwithin{equation}{section}

\newtheorem{theorem}{\bf Theorem}[section]

\theoremstyle{remark}

\def \R{{\mathbb R}}
\def\bmatrix#1{\left[\begin{matrix}
		#1
	\end{matrix}\right]}

\def \x{\bm{x}}

\def\bmatrix#1{\left[ \begin{matrix} #1 \end{matrix} \right]}

\def \R{{\mathbb R}}

\title{
Physics-Informed Broad Learning System: An Efficient Backpropagation-Free Framework for Solving Partial Differential Equations
}
\author{Pinki Khatun\footnotemark[2] \footnotemark[3]  \and  M. Sajid\footnotemark[4]  \and Abhinav Jha\footnotemark[2] \and M. Tanveer\footnotemark[4]  }
\date{}
\begin{document}
	\maketitle	
	\begin{abstract}
    Physics-informed neural networks (PINNs) have emerged as a powerful paradigm for solving partial differential equations (PDEs) by embedding governing physical laws into deep neural networks. However, their reliance on computationally expensive gradient-based optimization and deep architectures often results in slow training, high computational cost, and limited scalability. In this work, we propose a novel physics-informed broad learning system (PI-BLS), the first physics-informed learning framework based on broad RdNNs. The proposed formulation embeds the governing differential operator and the associated initial and boundary constraints directly into a linear output-layer optimization problem, thereby replacing nonlinear gradient-based training with a deterministic least-squares solution obtained via the pseudoinverse. Consequently, the entire learning process is reduced to a single linear optimization stage while preserving the underlying physical constraints. As a result, PI-BLS offers an efficient learning paradigm for a physics-informed learning framework for solving PDEs that eliminates iterative backpropagation while preserving the underlying physical constraints. Experimental results on representative forward PDE benchmarks demonstrate that PI-BLS achieves competitive and often superior performance with reduced training time and model parameters compared with conventional PINNs. 
\end{abstract}


     \footnotetext[2] {Department of Mathematics, Indian Institute of Technology Gandhinagar, Palaj, Gandhinagar, 382055, Gujarat, India (\tt{abhinav.jha@iitgn.ac.in}).}
      \footnotetext[3] {Department of Industrial Engineering, Università degli Studi di Firenze, Viale Morgagni 40/44, 50134 Firenze, Italy (\texttt{pinki.khatun@unifi.it}).}
         \footnotetext[4] {Department of Mathematics, Indian Institute of Technology Indore, Indore 453552, India (\texttt{phd2101241003@iiti.ac.in}, \texttt{mtanveer@iiti.ac.in}).}
     
     \section{Introduction}
Leveraging existing data for accurate modeling and prediction of partial differential equations (PDEs)-governed systems continues to be a challenging task, especially when dealing with ill-posed problems involving unknown parameters or boundary conditions, the presence of noise, and high computational demands.
%
%
%
The idea of employing neural networks as universal function approximators for solving PDEs dates back to the 1990s \cite{PINN1994}. Building upon this early concept, physics-informed machine learning has emerged over the past few decades as a powerful framework for addressing PDEs by integrating prior physical knowledge with available data \cite{PIML2021review}. A prominent realization of this paradigm is represented by physics-informed neural networks (PINNs), originally proposed in \cite{PINN1994, PINN1998}, which incorporate the governing physical laws directly into the loss function of neural networks.
PINNs extend the original idea by leveraging modern deep neural architectures, whose expressive power enables the representation of highly complex solution structures. Leveraging automatic differentiation \cite{AutoDifferential, paszke2017automatic}, the authors in \cite{PINN2019} introduced a practical framework for approximating solutions of PDEs in both forward and inverse settings. This development has been made possible by rapid advances in computational hardware, optimization algorithms, and large-scale training strategies, together with the widespread availability of automatic differentiation tools. 

 Being mesh-free and easy to implement, the PINN algorithm can be effectively applied to various classes of PDEs, including fractional PDEs \cite{pang2019fpinns}, integro-differential equations \cite{lu2021integro}, and stochastic PDEs \cite{zhang2020stochasticPDE, zhang2019stochasticPDE}. Consequently, PINNs have attracted significant attention and have been successfully applied to a wide range of scientific and engineering problems, such as fluid dynamics \cite{mao2020PINN}, heat transfer \cite{heat2021PINN}, biomedicine \cite{Biomedicin2020physics}, systems biology \cite{yazdani2020systems}, and electromagnetism \cite{electromagnetism_2023PINN}. By embedding the underlying PDEs as soft constraints within the learning process, PINNs enable neural networks to produce solutions that remain consistent with both observational data and fundamental physical principles, even in regimes where data are sparse or noisy. Since the development of PINNs, numerous extensions have been proposed to improve their accuracy, stability, and computational efficiency, such as gradient-enhanced PINNs \cite{wang2026gradient, yang2026gradient, gradientPINN}, extented PINN \cite{dd-xpinn2020, lee2025extended, zhu2025extended}, variational PINNs (vPINNs) \cite{berrone2022variational, eshaghi2025variational, los2025collocation} and hp-vPINNs \cite{hp-pinn2024}, which enhance solution stability by incorporating variational test spaces and domain decomposition-based schemes \cite{dd-xpinn2020, DD-PINN2020, wang2025exploring_DD}. See \cite{guo2025advances, review_hao2024pinnacle, Review2025} for the recent developments and comprehensive performance analysis.
 

 Despite their remarkable success, PINNs still face several fundamental challenges that limit their scalability and practical applicability. First, conventional PINNs rely on deep neural network architectures with a large number of trainable parameters, often resulting in complex optimization landscapes and slow convergence \cite{ krishnapriyan2021characterizing, wang2021understanding}. Second, training is performed using gradient-based iterative optimization algorithms, which are computationally expensive and frequently suffer from issues such as gradient pathologies and convergence to poor local minima, particularly for high-dimensional PDEs \cite{krishnapriyan2021characterizing, wang2021understanding,  wang2022when}. Third, the repeated computation of higher-order derivatives through automatic differentiation substantially increases both memory consumption and computational cost, making the training process highly dependent on modern GPU hardware and limiting its efficiency for large-scale scientific computing applications \cite{cuomo2022scientific, PINN2019}. Furthermore, the performance of PINNs is often sensitive to the selection of network architectures, loss-function weighting strategies, collocation-point sampling, and hyperparameter tuning, requiring considerable empirical effort to obtain reliable solutions \cite{cuomo2022scientific}. The aforementioned challenges motivate the search for efficient and scalable physics-informed learning frameworks that overcome the optimization and computational limitations of conventional PINNs while maintaining adherence to the underlying physical laws.

\section{Motivation and Contribution}
Randomized neural networks (RdNNs) constitute a class of feed-forward architectures \cite{RdNN1994, RdNN1992} developed to alleviate the computational burden and optimization difficulties commonly encountered in gradient-descent–based neural networks. The central idea behind these models is to introduce randomness into the network structure by fixing a subset of parameters, typically those associated with the hidden layers, while only a smaller set of parameters, usually in the output layer, is learned during training. As a result, the training process is significantly simplified, since the remaining parameters can often be obtained through efficient linear optimization procedures or even explicit closed-form solutions. By reducing the number of trainable parameters and avoiding expensive iterative backpropagation, RdNNs can achieve competitive performance with substantially lower computational cost and training time, while requiring far fewer computational resources.

Motivated by the need for computationally efficient learning algorithms, several alternative neural network paradigms have been proposed to improve training efficiency and scalability \cite{BLS2017, RdNN1994, wang2017stochastic}. Among these, the broad learning system (BLS) \cite{BLS2017} has emerged as a promising alternative to conventional deep neural networks. Built upon the principles of RdNNs, BLS replaces deep hierarchical architectures with a broad, flat network that expands horizontally rather than vertically. Specifically, it constructs expressive representations by integrating mapped feature nodes with nonlinear enhancement nodes, thereby enriching the feature space without increasing network depth \cite{chen2018universal}. Furthermore, the output weights are computed through efficient linear least-squares or pseudoinverse solutions, eliminating the need for computationally expensive iterative backpropagation. Consequently, BLS offers significantly faster training, lower computational complexity, and excellent scalability, making it an attractive candidate for scientific computing applications involving PDEs, where computational efficiency and rapid model adaptation are of paramount importance \cite{10530427, tanveer2025bls}. BLS expands the network in width, thereby avoiding many optimization challenges associated with deep models while maintaining strong approximation capability \cite{chen2018universal}. These characteristics naturally position BLS as a promising backbone for developing efficient physics-informed learning algorithms.

Despite these compelling advantages, existing BLS frameworks have been developed primarily for data-driven learning and do not explicitly incorporate the governing physical laws represented by PDEs. Consequently, they cannot be directly employed to solve forward PDE problems, where the predicted solution must simultaneously satisfy the governing equations and the prescribed initial and boundary conditions. This limitation reveals an important gap between efficient randomized learning models and physics-informed scientific machine learning.

Building upon the computational advantages of broad RdNNs, this paper proposes a novel \emph{physics-informed broad learning system (PI-BLS)} for solving forward PDEs. To the best of our knowledge, PI-BLS is the first physics-informed learning framework based on broad RdNNs for this class of problems. In the proposed framework, the PDE solution is approximated using a BLS architecture, where the feature and enhancement nodes are randomly generated and remain fixed throughout training. A physics-informed objective function is constructed by enforcing the governing PDE together with the prescribed boundary and initial conditions at collocation points. Consequently, the resulting learning problem is reformulated as a linear least-squares optimization problem with respect to the output-layer parameters, replacing the computationally expensive gradient-based optimization employed by conventional PINNs with an efficient closed-form learning strategy that completely eliminates backpropagation. This formulation yields a computationally efficient, scalable, and physically consistent framework for solving forward PDEs.





\vspace{2mm}
The main contributions of this work are summarized as follows:
\begin{itemize}

\item We propose a novel PI-BLS, which, to the best of our knowledge, is the first physics-informed learning framework based on broad RdNNs for solving forward PDEs.

\item We develop a new physics-informed learning formulation that embeds the governing PDE together with the associated boundary and initial conditions directly into the BLS, enabling the approximation of PDE solutions within a randomized broad neural architecture.

\item We reformulate the resulting learning problem as a linear least-squares optimization with respect to the output-layer parameters, allowing the model to be trained without iterative gradient-based backpropagation while preserving the underlying physical constraints.

\item We present a theoretical analysis of PI-BLS by deriving its stability property, residual error bound, sensitivity to perturbations, and physics consistency, providing rigorous guarantees for the robustness and physical fidelity of the proposed framework.

\item We evaluate the proposed PI-BLS framework on a diverse set of forward PDE benchmark problems, including one-dimensional steady-state, two-dimensional Poisson, and diffusion--reaction equations, and compare its performance against existing baselines in terms of solution accuracy, training losses, computational time, and model complexity.

\end{itemize}

The remainder of this paper is organized as follows.  Section \ref{SEC2} reviews the existing literature on BLS and PINNs.  Section \ref{SEC3} presents the proposed PI-BLS framework and its mathematical formulation.  Section \ref{Sec:theoretical_analysis} gives details of the theoretical properties and observations of the PI-BLS. Section \ref{SEC_experiment} reports numerical experiments and discusses the performance of proposed PI-BLS and comparison with the existing ones for several case studies.  Section \ref{SEC:Conclusion} provides the concluding remarks of the paper.

\medskip

\section{Brief Overview of BLS and PINN}\label{SEC2}
This section briefly introduces the BLS and the PINN architectures. BLS is a wide neural network model that enables efficient learning through a least-squares-based training strategy, while PINNs incorporate governing physical laws into the loss function to solve PDEs. These two architectures form the basis of the proposed approach.
%
Before proceeding, we introduce the following useful notation.

\medskip
\textbf{Notation:} We denote by $\mathbb{R}^{m \times n}$ the set of all $m \times n$ real matrices and by $\mathbb{R}^n$ the set of $n$-dimensional real vectors. 
The superscript $(\cdot)^T$ indicates the transpose of a matrix or vector.
By $\bmatrix{A & B}\in \R^{m\times (n+p)},$ we denote the concatenation of the matrices $A\in \R^{m\times n}$ and $B\in \R^{m\times p}.$ The symbol $I$ represents the identity matrix of appropriate dimension. 

\subsection{Broad Learning System}
In this subsection, we present the architecture of the BLS. Let $\{(X, Y)| X \in \mathbb{R}^{ N\times m}, \,Y \in \mathbb{R}^{ N}\}$ be the training dataset, where $X \in \mathbb{R}^{ N\times m}$ is the input matrix and $Y \in \mathbb{R}^{ N}$ is the output matrix.

The BLS is composed of three key components: (1) mapping feature nodes, (2) enhancement feature nodes, and (3) the output matrix. The training process of the BLS involves two main stages. First, the weights for both the mapping and enhancement nodes are randomly generated, eliminating the need for iterative tuning in these layers. Second, the weights connecting the hidden layer to the output layer are computed analytically, typically using a least-squares solution. In essence, only the final output weights are learned during training, while the rest of the network remains fixed. This significantly reduces computational cost and training time. 


\begin{figure}
    \centering
    \includegraphics[width=0.75\linewidth]{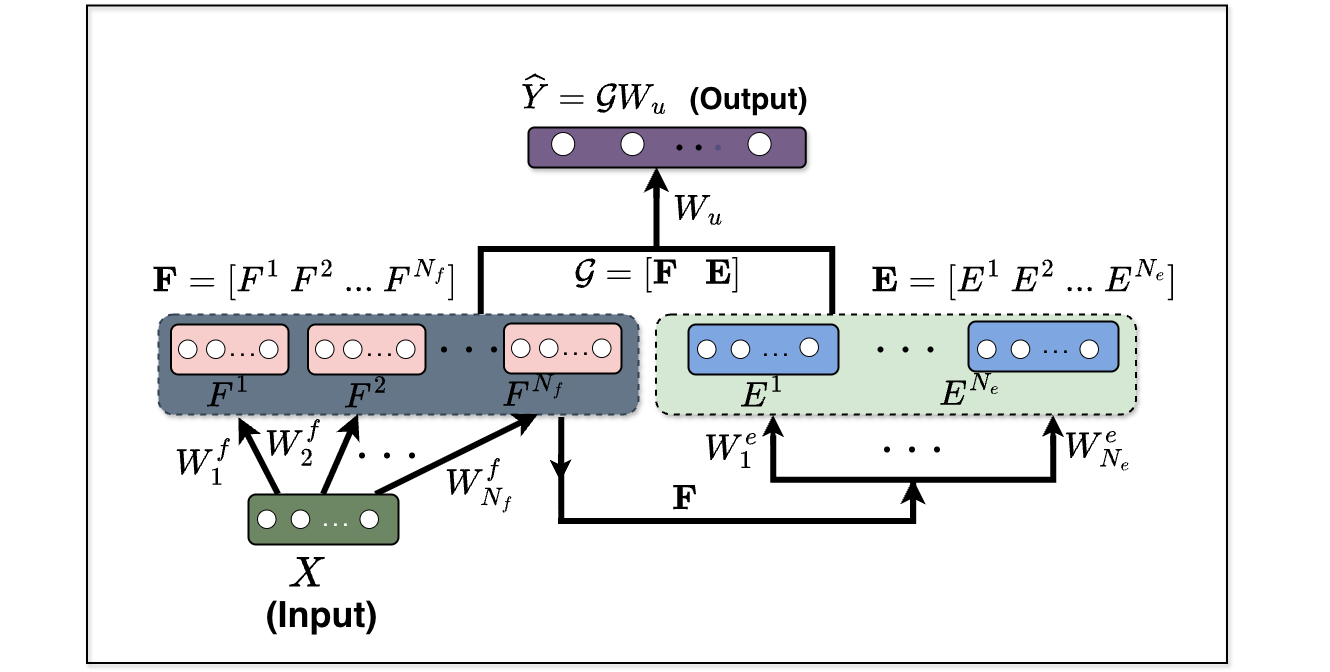}
    \caption{Architecture of the BLS}
    \label{fig:BLS}
\end{figure}

The architecture of the BLS \cite{BLS2017} is shown in Figure \ref{fig:BLS}. Let there be \( N_c \) feature layers, each consisting of \( p \) nodes. For the \( i^{\text{th}} \) feature layer, let \( W_i^f \in \mathbb{R}^{m \times p} \) and \( b_i^f\in \R^{1\times p} \) denote the randomly generated weight matrix and bias vector, respectively. The output of the \( i^{\text{th}} \) feature layer is then defined as:
\begin{align}
    F^i = \psi_i(XW_i^f + \bm{1}^T\otimes b_i^f) \in \mathbb{R}^{N \times p}, \quad i = 1, 2, \ldots, N_c,
\end{align}
where $\bm{1}\in \R^{N}$ is the vector of all ones and \( \psi_i \) is the activation function (feature mapping)  associated with the \( i^{\text{th}} \) feature group, applied elementwise to the matrix $\mathbf{F}W_i^f+b_i^f$.  The symbol $\otimes$ denotes the Kronecker product. The following augmented matrix gives the output of the feature group:
\begin{align}
    \mathbf{F}=\bmatrix{F^1& F^2& \ldots& F^{N_c}}\in \R^{N\times pN_c}.
\end{align}

In the next step, the augmented feature matrix is projected into the enhancement space through random transformations, followed by the application of activation functions. Let 
$N_e$ denote the number of enhancement layers, with each group consisting of 
$q$ nodes. Further, let \( W_i^e \in \mathbb{R}^{pN_c \times q} \) be the randomly generated weight matrix connecting the  feature matrix \( \mathbf{F} \) to the \( i^{\text{th}} \) enhancement group, and let \( b_i^e \in \R^{1\times q}\) denote the corresponding bias vector for the \( i^{\text{th}} \) enhancement group.
Then
\begin{align}
    E^i = \phi_i(\mathbf{F}W_i^e + \bm{1}\otimes b_i^e) \in \mathbb{R}^{N \times q}, \quad i = 1, 2, \ldots, N_e,
\end{align}
where $\phi_i$ is the activation function used to generate nonlinear
features, applied elementwise to the matrix $\mathbf{F}W_j^e+b_j^e$. The output of the $N_e$ enhancement feature groups is 
\begin{align}
    \mathbf{E}=\bmatrix{E^1& E^2& \ldots& E^{N_e}}\in \R^{N\times qN_e}.
\end{align}
Finally, the resulting enhancement matrix \(\mathbf{E} \), together with the augmented feature matrix \( \mathbf{F} \), is passed to the output layer, where the resulting output of the network is computed as follows:
\begin{align}
    \widehat{Y}=\mathcal{G}W_{u},
\end{align}
 where $\mathcal{G}=\bmatrix{\mathbf{F}& \mathbf{E}}\in \R^{N\times (pN_c+qN_e)}$ denotes the concatenated matrix of the augmented feature layer and the resultant enhancement layer, and $W_{u}\in \R^{(pN_c+qN_e)\times 1}$
  is the randomly generated weight matrix that connects this combined representation to the output layer.
Finally, the training of the BLS is completed by estimating the output weight matrix through the following least-squares optimization problem:
\begin{equation}
\min_{W_u} \; \|\mathcal{G}W_u-Y\|_2^2.
\end{equation}
\subsection{Physics-Informed Neural Network}

Physics-informed neural networks (PINNs) are a class of machine learning methods that incorporate domain-specific knowledge, typically represented by PDEs—into the neural network training process. The central idea of PINNs is to embed the governing physical laws as constraints, ensuring that the neural network's predictions adhere to the underlying system dynamics during training.
\begin{figure}
    \centering
    \includegraphics[width=0.65\linewidth]{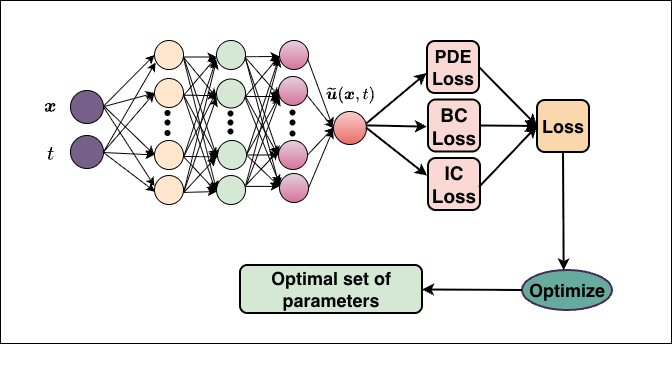}
    \caption{Framework of PINN}
    \label{fig:PINN}
\end{figure}
We consider the following  form of a general  PDE:
\begin{align}\label{pde1}
   & \frac{\partial}{\partial t} \bm{u}(\x,t)+\mathcal{N}_{\bm{x}}[\bm{u}(\bm{x},t)]=R(\bm{x},t),~~ \bm{x}\in \Omega, ~\text{and}~ t\in[0,T]\\ \label{pde2}
   & \mathcal{B}[\bm{u}(\bm{x},t)]=B(\bm{x},t), ~~\bm{x}\in \partial \Omega, ~\text{and}~t\in [0,\, T],\\ \label{pde3}
   & \bm{u}(\bm{x},0)=\mathcal{I}(\bm{x}), ~\bm{x}\in\Omega, 
\end{align}
where $\bm{x}\in \R^d$, $t\in [0,\, T],$ and $\mathcal{N}_{\x}[\bm{u}]$ and $\mathcal{B}[\bm{u}]$ are the potentially nonlinear differential operator and the boundary condition operator, respectively, and $\partial \Omega$ is the boundary of the computation domain $\Omega.$

The framework for PINN is given in Figure \ref{fig:PINN}.  The PINN is employed to approximate the 
  solution $\bm{u}(x,t)$ by $\widetilde{\bm{u}}(x,t)$, where the network takes $(\bm{x}, t)$ as inputs. Let the collocation points at the interior of the domain be denoted by $\{\bm{x}_{i,u}, t_{i,u}\}_{i=1}^{N_c}$, the boundary points by $\{\bm{x}_{i, b}, t_{i, b}\}_{i=1}^{N_b}$, and the initial points $\{\bm{x}_{i, I}, 0\}_{i=1}^{N_I}$, where $N_{u},$ $N_{b}$ and $N_{I}$ denote the number of training collocation points, boundary points, and initial condition points, respectively. The residuals associated with the PDE, boundary conditions, and initial conditions are denoted by $\mathcal{J}_{u},$ $\mathcal{J}_{b},$ and $\mathcal{J}_{I},$ respectively, and given as follows:
  \begin{align}\label{erro_PDE}
&\mathcal{J}_{u}(\bm{x}, t)=\frac{\partial \widetilde{\bm{u}}}{\partial t}(\bm{x}, t)+ \mathcal{N}_{\bm x}[\widetilde{\bm{u}}](\bm{x}, t)-{R}(\bm{x}, t)~~\text{on}~~\{\bm{x}_{i,u}, t_{i,u}\}_{i=1}^{N_c}, \\ \label{error_BC}
&\mathcal{J}_{b}(\bm{x}, t)=\mathcal{B}[\widetilde{\bm{u}}](\bm{x}, t)-{B}(\bm{x}, t)~~ \text{on}~~\{\bm{x}_{i,b}, t_{i,b}\}_{i=1}^{N_b},\\ \label{error_IC}
&\mathcal{J}_{I}(\bm{x}, 0)=\widetilde{\bm{u}}(\bm{x},0)-\mathcal{I}(\bm{x})~~\text{on}~~ \{\bm{x}_{i,I}, 0\}_i^{N_I}.
  \end{align}
Therefore, the loss terms due to the PDE, boundary conditions, and initial conditions, respectively,  are given by \begin{align}\label{erro1_PDE}
&\mathcal{J}_{u}=\sum_{i=1}^{N_c}\left|\frac{\partial \widetilde{\bm{u}}}{\partial t}(\bm{x}_{i,u}, t_{i,u})+ \mathcal{N}_{\bm x}[\widetilde{\bm{u}}(\bm{x}_{i,u}, t_{i,u})]-R(\bm{x}_{i,u}, t_{i,u}) \right|^2, \\ \label{error1_BC}
&\mathcal{J}_{b}=\sum_{i=1}^{N_b}\left|\mathcal{B}[\widetilde{\bm{u}}](\bm{x}_{i,b}, t_{i,b})-B(\bm{x}_{i,b}, t_{i,b})\right|^2,\\ \label{error1_IC}
&\mathcal{J}_{I}=\sum_{i=1}^{N_I}\left|\widetilde{\bm{u}} (\bm{x}_{i,I}, 0)-\mathcal{I}(\bm{x}_{i,I})\right|^2.
  \end{align}
 Accordingly, the total loss to be minimized for a PINN is given by:
  \begin{align}
      \mathcal{J}:=\frac{1}{2N_{u}} \mathcal{J}_{u}+\frac{1}{2N_{b}} \mathcal{J}_{b}+\frac{1}{2N_{I}} \mathcal{J}_{b}.
  \end{align}
   The partial derivatives involved in the loss function can be efficiently computed using automatic differentiation \cite{AutoDifferential, paszke2017automatic}.
By applying a gradient-based optimization algorithm, we can minimize the loss function $\mathcal{J}$ to determine the optimal weights and biases of the neural network.

\section{Proposed Physics-Informed Broad Learning System}\label{SEC3}
In this section, we introduce the physics-informed broad learning system (PI-BLS), which integrates the principles of PINNs with BLS to solve linear PDEs.  The core idea is to incorporate physical laws directly into the BLS framework as part of the cost function. This transforms linear PDE problems into a regularized linear least-squares problem, allowing the solution to simultaneously satisfy the governing equations and boundary/initial conditions in a computationally efficient manner. In the following, we discuss the mathematical architecture of PI-BLS.

Consider the general PDE problem \eqref{pde1}-\eqref{pde3}. The architecture of the PI-BLS  is shown in Figure \ref{fig:PI-BLS}.
Define $X=[\bm{x}^T,\,t]\in \R^{1 \times (d+1)},$ where $\bm{x}\in \R^{d}$ and $t\in \R$. Let the network have $N
_f$ feature layers with $p$ feature nodes in each layer. The feature layers are denoted by $F^i,$ $i=1,2,\ldots, N_c$. Further, let $W^f_i\in \R^{(d+1)\times p}$ be the weight matrix and $b^f_i\in \R^{1\times p}$ be the biases in the $i^{th}$ feature layer. Then the output of the feature layer segment is given by
\begin{align}
    \textbf{F}=\bmatrix{F^1& F^2& \ldots& F^{N_c}}\in \R^{1 \times pN_c},
\end{align}
where $F^i(X)=\psi_i(X W^f_{i}+b_i^f)\in \R^{1\times p}$ is the output in the  $i^{th}$ physics-informed feature layer, and $\psi_i: \mathbb{R}\rightarrow \mathbb{R}$ are the feature maps or nonlinear activation functions. 
\begin{figure}
    \centering
    \includegraphics[width=0.8\linewidth]{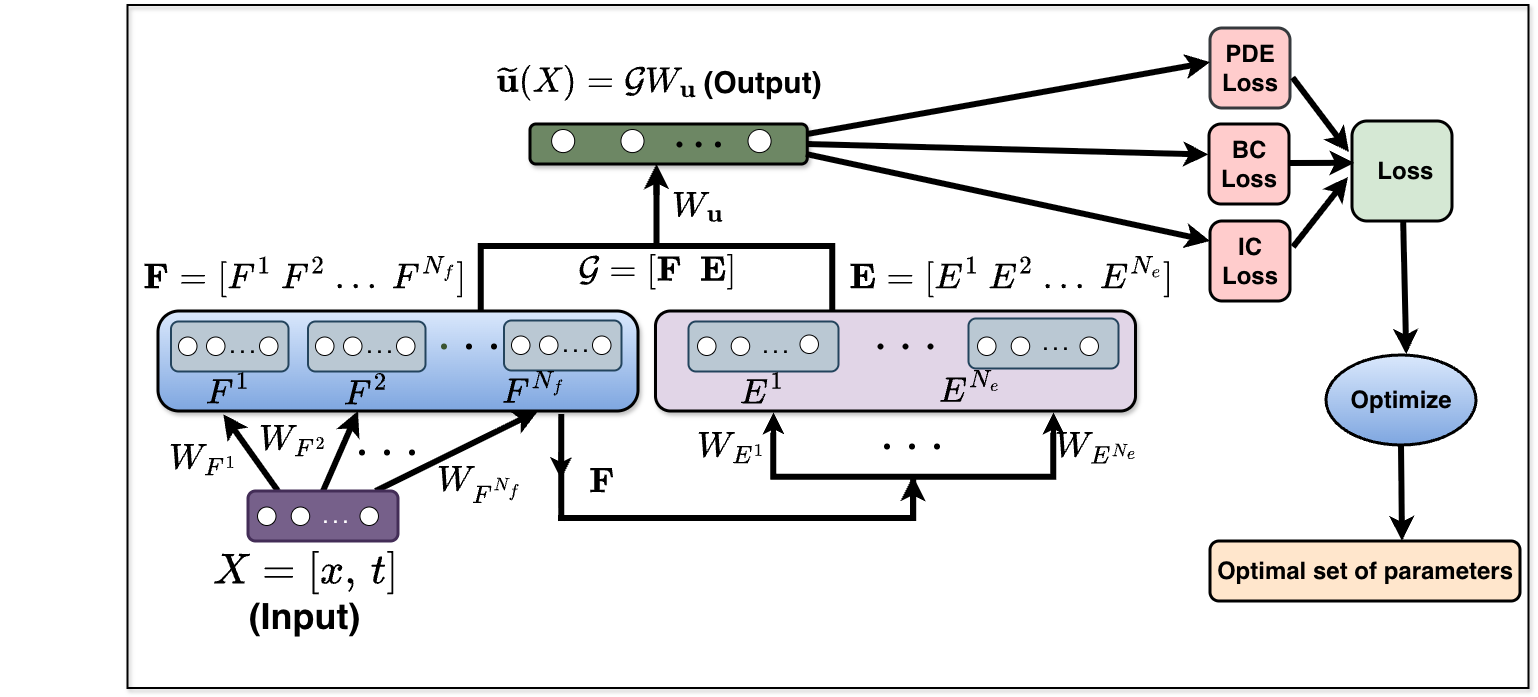}
    \caption{Framework of the proposed PI-BLS}
    \label{fig:PI-BLS}
\end{figure}

Let the network consist of \( N_e \) enhancement layers, each containing \( q \)  enhancement nodes. These enhancement layers are denoted by \( E^j \), for \( j = 1, 2, \ldots, N_e \). Let \( W^e_j \in \mathbb{R}^{pN_c \times q} \) be the weight matrix and \( b^e_j \in \mathbb{R}^{1 \times q} \) be the bias vector in the \( j^{\text{th}} \) enhancement layer. Then, the output of the enhancement layer is given by:

\begin{align}
   \textbf{E}=\bmatrix{E^1& E^2& \ldots& E^{N_e}}\in \R^{1 \times qN_e},
\end{align}
where
\[
E^j(X)=\phi_j(\mathbf{F}W_j^e+b_j^e)\in\mathbb{R}^{1\times q},
\]
with $\phi_j:\mathbb{R}\rightarrow\mathbb{R}$ being  nonlinear activation functions.
The final output of the network denoted by $\widetilde{\bm{u}}$ and is given by 
\begin{align}
   \widetilde{\bm{u}}(X; W_u)= \mathcal{G}W_{ \bm u},
\end{align}
where $\mathcal{G}=[\bm{F}\,~ \bm{E}]\in \R^{1 \times (pN_c+qN_e)}$ and $W_{\bm u }\in \R^{(pN_c+qN_e) \times 1}.$ Hence, $\widetilde{\bm{u}}(X; W_u)$ is the   approximate solution  of the PDE \eqref{pde1}--\eqref{pde3} obtained using the network.


Let $N_{u},$ $N_b$, and $N_I$ denote the numbers of collocation points, boundary points, and initial conditions, respectively, due to training. The error functions in approximating the PDE, boundary conditions, and initial conditions are the same as in \eqref{erro_PDE}, \eqref{error_BC}, and \eqref{error_IC}. 
Hence, the loss functions to be minimized for a PI-BLS are given by:
\begin{align}
&\mathcal{J}_{u}(\bm{x}, t)=\frac{\partial (\mathcal{G}W_u)}{\partial t} (\bm{x}, t)+ \mathcal{N}_{\bm x}[\mathcal{G}W_u](\bm{x}, t)-{R}(\bm{x}, t)~~\text{on}~~\{\bm{x}_{i, u}, t_{i, u}\}_{i=1}^{N_c}, \\ 
&\mathcal{J}_{b}(\bm{x}, t)=\mathcal{B}[\mathcal{G}W_u](\bm{x}, t)-{B}(\bm{x}, t)~~ \text{on}~~\{\bm{x}_{i,b}, t_{i,b}\}_{i=1}^{N_b},\\ 
&\mathcal{J}_{I}(\bm{x})=\mathcal{G}W_u(\bm{x},0)-\mathcal{I}(\bm{x})~~\text{on}~~ \{\bm{x}_{i,I}, 0\}_i^{N_I}.
  \end{align}
For convenience, we define the corresponding data matrices as: \[X_{\mathrm{col}}= [(\bm{x}_{1,u}, t_{1, u}), (\bm{x}_{2, u}, t_{2, u}), \ldots, (\bm{x}_{N_{ u}, u}, t_{N_{u}, u})]^T, \]
\[X_{b}= [(\bm{x}_{1, b}, t_{1, b}), (\bm{x}_{2, b}, t_{2, b}), \ldots, (\bm{x}_{N_{b}, b}, t_{N_{b}, b})]^T, \]
and \[X_{I}= [(\bm{x}_{1, I}, 0), (\bm{x}_{2,I}, 0), \ldots, (\bm{x}_{N_{I}, I}, 0)]^T.\]
  \begin{algorithm}
\caption{Training process of proposed PI-BLS for PDE problems}
\label{alg:1}

\begin{algorithmic}[1]
\State \textbf{Input:} Dimension of the the input $X,$ number of feature groups $N_c,$ number of feature nodes $p,$ number of enhancement groups $N_e,$ number of enhancement nodes $q,$ hyperparameters $\lambda_b,\, \lambda_I, \lambda_{\mathrm{reg}}>0$
\State \textbf{Output:} Output weight matrix $W^*_{u}$/approximate solution $\widetilde{\bm{u}}(X; W^*_u).$

\State \textbf{Generate} collocation points $X_{\mathrm{col}},$ boundary points $X_{b}$ and initial points $X_{I}.$
\State \textbf{Feature Layer Construction:}
\For{$i = 1$ to $N_c$}
    \State Generate random weight matrix $W_i^{f} \in \mathbb{R}^{d \times p}$ and bias vector ${b}_i^{f} \in \mathbb{R}^{1 \times p}$.
    \State Compute feature mapping:
    \(
    F^{i} = \psi_i(XW_i^{f} + {b}_i^{f}).
    \)
\EndFor
\State Concatenate feature outputs
\(
\mathbf{F} = \bmatrix{
F^1 & F^2 & \cdots & F^{N_c}
}.
\)

\State \textbf{Enhancement Layer Construction:}
\For{$i = 1$ to $N_e$}
    \State Generate random weight matrix $W_i^{e} \in \mathbb{R}^{(pN_c) \times q}$ and bias vector $b_i^{e} \in \mathbb{R}^{1 \times q}$.
    \State Compute enhancement mapping: 
    \(
    \mathbf{E}^{i} = \phi_i(\mathbf{F}{W}_i^{e} + b_i^{e}).
    \)
\EndFor
\State Concatenate enhancement outputs
\(
\mathbf{E} = \bmatrix{
E^1 &E^2 & \cdots & E^{N_e}
}.
\)
\State Compute $\mathcal{G}=\bmatrix{\bm{F} & \bm{E}}$
\State Compute \begin{align*}
H_{\mathrm{col}} &= \frac{\partial \mathcal{G}}{\partial t} (X_{\mathrm{col}}) + \mathcal{N}_{\bm{x}}[\mathcal{G}](X_{\mathrm{col}}), ~~H_b = \mathcal{B}[\mathcal{G}](X_b), ~~H_I = \mathcal{G}(X_I).
\end{align*}
\State Construct the matrices \[\mathcal{H} = \bmatrix{H_{\text{col}} \\ \sqrt{\lambda_b}H_{b}\\ \sqrt{\lambda_I} H_{I}}~~ \text{and}~~ Z=\bmatrix{R_{\text{col}} \\ \sqrt{\lambda_b}R_{b}\\ \sqrt{\lambda_I} R_{I}}.\]
\State Compute output layer weight matrix $W_u$ using \eqref{eq:Wu}. 
\end{algorithmic}
\end{algorithm}

 When $\mathcal{N}_{\bm{x}}[\,\cdot \,]$ and  $\mathcal{B}[\,\cdot\,]$  are linear differential operators, then the proposed minimization problem is expressed as:
\begin{align}
    \arg\min_{W_u}~ \Bigg(&\|H_{\text{col}} W_u-R(X_{\text{col}})\|_2^2+ \lambda_b\, \|H_{b} W_u- B(X_b)\|^2_2 \\
    &+\lambda_I \,\|H_{I} W_u- \mathcal{I}(X_I)\|^2_2 +\lambda_{\mathrm{reg}}\, \|W_u\|_2^2\Bigg),
\end{align}
where 
\begin{align}
H_{\mathrm{col}} &= \frac{\partial \mathcal{G}}{\partial t} (X_{\mathrm{col}}) + \mathcal{N}_{\bm{x}}[\mathcal{G}](X_{\mathrm{col}})\in \R^{N_{u} \times (pN_c +q N_e)}, \\
H_b &= \mathcal{B}[\mathcal{G}](X_b)\in \R^{N_{b} \times (pN_c +q N_e)}, \\
H_I &= \mathcal{G}(X_I) \in \R^{N_{I} \times (pN_c +q N_e)}.
\end{align}
The hyperparameters $\lambda_b,\, \lambda_I, \lambda_{\mathrm{reg}}>0$  are scalar weights used to balance the loss function.
The above optimization problem can be compactly expressed as:
\begin{align}
   W^*_u =  \arg\min_{W_u} ~ \|\mathcal{H} W_u- Z\|_2^2 \,+ \, \lambda_{\mathrm{reg}}\, \|W_u\|_2^2,
\end{align} 
where   \[\mathcal{H}=\bmatrix{H_{\text{col}} \\ \sqrt{\lambda_b}H_{b}\\ \sqrt{\lambda_I} H_{I}} \in \R^{N_u\times (pN_c+ qN_e)} ~\text{and} ~Z =\bmatrix{R_{\text{col}} \\ \sqrt{\lambda_b}R_{b}\\ \sqrt{\lambda_I} R_{I}} \in \R^{N_u},\] 
and $N_u=(N_c+N_b +N_I ).$ The matrices $\mathcal{H} $ and $Z$ depend upon the physics law of PDE, boundary conditions, and initial conditions. For this optimal weight matrix in the output layer, PI-BLS approximate solution is defined as $\widetilde{\bm u}(X; W^*_u).$

  %
  The weight matrix $W_u$ at the output layer can be calculated by the least squares method and is given by 
  \begin{align}\label{eq:Wu}
  W^*_u=\left\{\begin{array}{cc}
    (\mathcal{H}^T\mathcal{H} +\lambda_{\mathrm{reg}} I)^{-1} \mathcal{H}^T Z,   & N\leq (pN_c + qN_q), \\[4pt]
       \mathcal{H}^T (\mathcal{H} \mathcal{H}^T +\lambda_{\mathrm{reg}} I)^{-1} \mathcal{H}^T Z, & N> (pN_c + qN_q).
  \end{array}\right.
  \end{align}
Algorithm  \ref{alg:1} gives the step-by-step implementation of the training process for a general linear PDE. We use automatic differentiation \cite{AutoDifferential, paszke2017automatic} to find the derivative of the approximate function.

\subsection{Computational Complexity Analysis}

\begin{table}[!t]
\centering
\caption{Computational complexity comparison of different physics-informed learning frameworks. }
\label{tab:complexity}
\resizebox{12cm}{!}{
\begin{tabular}{lll}
\hline
\textbf{Method} & \textbf{Training Strategy} & \textbf{Computational Complexity} \\
\hline
PINN &
Gradient-based optimization &
$\mathcal{O}\!\left(TN_uL_hN_h^2\right)$ \\

PI-BLS &
Regularized least-squares &
$\mathcal{O}\!\left(N_uM^2+M^3\right)$ \\

PI-ELM &
Least-squares &
$\mathcal{O}\!\left(N_u N_h^2+N_h^3\right)$ \\

B-PI-ELM &
Bayesian least-squares &
$\mathcal{O}\!\left(N_u N_h^2+N_h^3\right)$ \\
\hline
\end{tabular}}
\end{table}

The computational complexity of the proposed PI-BLS framework is compared with those of PINN, and  RdNN-based physics-informed extreme learning machine (PI-ELM) \cite{PIELM2020}, and Bayesian PI-ELM (B-PI-ELM) \cite{PINN2023bayesian}. Let $N_u$ denote the total number of collocation points, including interior, boundary, and initial points. For PINN, let $L_h$ and $N_h$ denote the number of hidden layers and hidden nodes per layer, respectively, and let $T$ denote the total number of optimization iterations. For PI-BLS, let $M=pN_c+q,$ where $N_e=1$, $N_c\times p$ represents the total number of mapped feature nodes, and $q$ denotes the number of enhancement nodes.
For a conventional PINN, the dominant computational cost arises from repeated forward propagation and backpropagation over all trainable parameters during each optimization iteration. Assuming a fully connected network, the overall training complexity can be approximated as
$\mathcal{O}_{\text{PINN}}
=
\mathcal{O}\!\left(TN_uL_h N_h^2\right)$.
In contrast, the hidden-layer parameters of PI-BLS remain fixed after random initialization, and only the output-layer weights are estimated through a regularized least-squares problem. The computational cost consists of constructing the hidden-layer output matrix, forming the normal equations, and solving the resulting linear system, yielding
$\mathcal{O}_{\text{PI-BLS}}
=
\mathcal{O}\!\left(N_uM^2+M^3\right)$.
Similarly, both PI-ELM  and B-PI-ELM  employ randomized hidden layers and determine the output weights through linear least-squares optimization. Their dominant computational complexity is therefore given by
$\mathcal{O}_{\text{PI-ELM}}
=
\mathcal{O}_{\text{B-PI-ELM}}
=
\mathcal{O}\!\left(N_uN_h^2+ N_h^3\right)$.

The computational complexities of all competing methods are summarized in Table~\ref{tab:complexity}. It can be observed that the complexity of PINN scales linearly with the number of optimization iterations, making the training process considerably more expensive for large-scale problems. In contrast, PI-BLS, PI-ELM, and B-PI-ELM require only a single least-squares optimization, eliminating iterative backpropagation altogether. Consequently, PI-BLS achieves computational complexity of the same order as PI-ELM and B-PI-ELM while maintaining a richer broad learning architecture.

\section{Theoretical Properties of PI-BLS} \label{Sec:theoretical_analysis}
This section establishes the theoretical foundations of the proposed PI-BLS framework. Specifically, we analyze its numerical stability, error bound, sensitivity to perturbations, and consistency with the underlying governing physics. Together, these theoretical results demonstrate the soundness of the proposed formulation and provide rigorous guarantees for its application to forward PDE problems.
\begin{theorem}[\textbf{Stability of the PI-BLS Solution}]
Let $\mathcal{H}\in\mathbb{R}^{N_u\times M}$ denote the hidden-layer output matrix of the proposed PI-BLS. Let ${W}_{u}\in\mathbb{R}^{M}$ denote the output-weight vector obtained by minimizing the regularized physics-informed objective
\begin{equation}
\mathcal{J}({W}_{u})
:=
\left\|
\mathcal{H}{W}_{u}
-
Z
\right\|_{2}^{2}
+
\lambda_{\mathrm{reg}}
\|
{W}_{u}
\|_{2}^{2},
\end{equation}
where $Z\in\mathbb{R}^{N_u}$ is the target vector containing the PDE residuals together with the prescribed boundary and initial conditions, and
$\lambda_{\mathrm{reg}}>0$ is the regularization parameter.

Then the PI-BLS solution is stable in the sense that
\begin{equation}
\|
{W}_{u}
\|_{2}
\le
\frac{\|
\mathcal{H}
\|_{2}}
{\lambda_{\mathrm{reg}}}
\,
\|
Z
\|_{2}.
\end{equation}
Hence, every bounded physics-informed target vector produces a bounded output-weight vector.
\end{theorem}
\begin{proof}
The objective function is given by
\begin{equation}
\mathcal{J}({W}_{u})
=
\|
\mathcal{H}{W}_{u}
-
Z
\|_{2}^{2}
+
\lambda_{\mathrm{reg}}
\|
{W}_{u}
\|_{2}^{2}.
\end{equation}
Since $\mathcal{J}({W}_{u})$ is continuously differentiable, its minimizer satisfies the first-order optimality condition
\begin{equation}
\nabla \mathcal{J}({W}_{u})
=
2\mathcal{H}^{T}
(\mathcal{H}{W}_{u}-Z)
+
2\lambda_{\mathrm{reg}}{W}_{u}
=
0.
\end{equation}
Rearranging the above equation gives
\begin{equation}
(\mathcal{H}^{T}\mathcal{H}
+
\lambda_{\mathrm{reg}}I)
{W}_{u}
=
\mathcal{H}^{T}Z,
\end{equation}
and therefore
\begin{equation}
{W}_{u}
=
(\mathcal{H}^{T}\mathcal{H}
+
\lambda_{\mathrm{reg}}I)^{-1}
\mathcal{H}^{T}Z.
\end{equation}
Taking the Euclidean norm on both sides yields
\begin{align}
\|
{W}_{u}
\|_{2}
&\le
\|
(\mathcal{H}^{T}\mathcal{H}
+
\lambda_{\mathrm{reg}}I)^{-1}
\|_{2}
\,
\|
\mathcal{H}^{T}
\|_{2}
\,
\|
Z
\|_{2}.
\end{align}
Since
\(
\lambda_{\min}
(\mathcal{H}^{T}\mathcal{H}
+
\lambda_{\mathrm{reg}}I)
\ge
\lambda_{\mathrm{reg}},
\)
we have
\begin{equation}
\|
(\mathcal{H}^{T}\mathcal{H}
+
\lambda_{\mathrm{reg}}I)^{-1}
\|_{2}
\le
\frac{1}{\lambda_{\mathrm{reg}}}.
\end{equation}
Moreover,
\(
\|
\mathcal{H}^{T}
\|_{2}
=
\|
\mathcal{H}
\|_{2}.
\)
Hence,
\begin{equation}
\|
{W}_{u}
\|_{2}
\le
\frac{\|
\mathcal{H}
\|_{2}}
{\lambda_{\mathrm{reg}}}
\,
\|
Z
\|_{2},
\end{equation}
which proves the desired result. \end{proof}

\begin{theorem}[\textbf{Residual Error Bound}]
Let \({W}_{u}^{*}
=
\underset{{W}_{u}}{\arg\min}
\mathcal{J}({W}_{u})\),
where
\begin{equation}
\mathcal{J}({W}_{u})
=
\left\|
\mathcal{H}{W}_{u}
-
Z
\right\|_{2}^{2}
+
\lambda_{\mathrm{reg}}
\|
{W}_{u}
\|_{2}^{2},
\end{equation}
with $\lambda_{\mathrm{reg}}\ge0$.
Then, for every admissible output-weight vector
${W}_{u}$,
the optimal PI-BLS solution satisfies
\begin{equation}
\mathcal{J}({W}_{u}^{*})
\le
\mathcal{J}({W}_{u}).
\end{equation}
Furthermore, if there exists an exact output-weight vector
${W}_{u}^{e}$
such that $\mathcal{H}{W}_{u}^{e}
=
Z$,
then the approximation error satisfies
\begin{equation}
\|
\mathcal{H}
({W}_{u}^{*}
-
{W}_{u}^{e})
\|_{2}
\le
\sqrt{
\mathcal{J}({W}_{u}^{*})
}.
\end{equation}
\end{theorem}

\begin{proof}
Since
${W}_{u}^{*}$
is the minimizer of the objective function,
it follows immediately that
\[
\mathcal{J}({W}_{u}^{*})
\le
\mathcal{J}({W}_{u})
\]
for every admissible
${W}_{u}$. Suppose there exists an exact solution
${W}_{u}^{e}$
satisfying $\mathcal{H}{W}_{u}^{e}
=
Z$.
Then
$\mathcal{H}
({W}_{u}^{*}
-
{W}_{u}^{e})
=
\mathcal{H}{W}_{u}^{*}
-
Z.$
Therefore,
\begin{align}
\|
\mathcal{H}
({W}_{u}^{*}
-
{W}_{u}^{e})
\|_{2}^{2}
~=~
\|
\mathcal{H}{W}_{u}^{*}
-
Z
\|_{2}^{2} ~\le~
\mathcal{J}({W}_{u}^{*}),
\end{align}
since the regularization term
$\lambda_{\mathrm{reg}}
\|
{W}_{u}^{*}
\|_{2}^{2}$
is non-negative. Taking square roots on both sides completes the proof. \end{proof}


\begin{theorem}[\textbf{Sensitivity of the PI-BLS Solution}]
Assume that the output-weight vector is computed by solving
\begin{equation}
{W}_{u}
=
(\mathcal{H}^{T}\mathcal{H}
+
\lambda_{\mathrm{reg}}I)^{-1}
\mathcal{H}^{T}Z,
\end{equation}
where
$\lambda_{\mathrm{reg}}>0$. Let the physics-informed target vector be perturbed as $\widetilde{Z}
=
Z
+
\Delta Z$, and let
$\widetilde{W}_{u}$ denote the corresponding PI-BLS solution. Then the perturbation in the learned output weights satisfies
\begin{equation}
\|
\widetilde{W}_{u}
-
{W}_{u}
\|_{2}
\le
\frac{\|
\mathcal{H}
\|_{2}}
{\lambda_{\mathrm{reg}}}
\,
\|
\Delta Z
\|_{2}.
\end{equation}
Hence, sufficiently small perturbations in the physics-informed target vector produce proportionally small variations in the learned PI-BLS solution. 
\end{theorem}

\begin{proof} The perturbed solution is given by
\begin{equation}
\widetilde{W}_{u}
=
(\mathcal{H}^{T}\mathcal{H}
+
\lambda_{\mathrm{reg}}I)^{-1}
\mathcal{H}^{T}
\widetilde{Z}.
\end{equation}
Subtracting the original solution yields
\begin{align}
\widetilde{W}_{u}
-
{W}_{u}
=
(\mathcal{H}^{T}\mathcal{H}
+
\lambda_{\mathrm{reg}}I)^{-1}
\mathcal{H}^{T}
(
\widetilde{Z}
-
Z
)=
(\mathcal{H}^{T}\mathcal{H}
+
\lambda_{\mathrm{reg}}I)^{-1}
\mathcal{H}^{T}
\Delta Z.
\end{align}
Taking the Euclidean norm gives
\begin{align}
\|
\widetilde{W}_{u}
-
{W}_{u}
\|_{2}
&\le
\|
(\mathcal{H}^{T}\mathcal{H}
+
\lambda_{\mathrm{reg}}I)^{-1}
\|_{2}
\,
\|
\mathcal{H}^{T}
\|_{2}
\,
\|
\Delta Z
\|_{2}.
\end{align}
Since
\[
\|
(\mathcal{H}^{T}\mathcal{H}
+
\lambda_{\mathrm{reg}}I)^{-1}
\|_{2}
\le
\frac{1}{\lambda_{\mathrm{reg}}}
~~\text{and}~~
\|
\mathcal{H}^{T}
\|_{2}
=
\|
\mathcal{H}
\|_{2},
\]
it follows that
\[
\|
\widetilde{W}_{u}
-
{W}_{u}
\|_{2}
\le
\frac{\|
\mathcal{H}
\|_{2}}
{\lambda_{\mathrm{reg}}}
\,
\|
\Delta Z
\|_{2},
\]
which completes the proof. \end{proof}


\begin{theorem}[\textbf{Physics Consistency of PI-BLS}]
Let $\widetilde{\bm u}(X;{W}_u)$ denote the PI-BLS approximation of the solution $\bm{u}(X)$, where ${W}_u$ represents the trainable output weights. Suppose that ${W}_u^{*}$ minimizes the physics-informed objective function
\begin{equation}
\mathcal{J}
=
\mathcal{J}_{u}
+
\lambda_b\mathcal{J}_{b}
+
\lambda_I\mathcal{J}_{I}
+
\lambda_{\mathrm{reg}}
\|{W}_u\|_2^2,
\end{equation}
with $\lambda_b,\lambda_I,\lambda_{\mathrm{reg}}\ge0$. If the regularization parameter satisfies
$\lambda_{\mathrm{reg}}=0$
and the minimum value of the objective function is zero, i.e., $\mathcal{J}({W}_u^{*})=0$, then the PI-BLS approximation exactly satisfies the governing PDE together with all prescribed boundary and initial conditions at every collocation point.
\end{theorem}

\begin{proof} The objective function is given by
\[
\mathcal{J}
=
\mathcal{J}_{u}
+
\lambda_b\mathcal{J}_{b}
+
\lambda_I\mathcal{J}_{I}
+
\lambda_{\mathrm{reg}}
\|{W}_u\|_2^2,
\]
where each loss term is non-negative because it is defined as the mean squared residual of the corresponding physical constraint. Hence,
$\mathcal{J}_{u}\ge0,~
\mathcal{J}_{b}\ge0,~\text{and}~
\mathcal{J}_{I}\ge0$.

Assume that $\lambda_{\mathrm{reg}}=0$ and that the global minimum satisfies $\mathcal{J}({W}_u^{*})=0$. Since the objective is a sum of non-negative terms, each term must vanish individually, namely,
$\mathcal{J}_{u}=0,~
\mathcal{J}_{b}=0,~\text{and}~
\mathcal{J}_{I}=0$. The condition $\mathcal{J}_{u}=0$ implies that the residual of the governing differential equation is zero at every interior collocation point, i.e.,
$\frac{\partial \widetilde{\bm u}(X; {W}_u^{*})}{\partial t} + \mathcal{N}_{\bm x}
\bigl(
\widetilde{\bm u}(X;{W}_u^{*})
\bigr)
=
R(X).$
Similarly, $\mathcal{J}_{b}=0$ implies that all prescribed boundary conditions are satisfied exactly on the boundary collocation points, while $\mathcal{J}_{I}=0$ ensures that the initial conditions are satisfied exactly at the temporal collocation points whenever the problem is time-dependent.

Therefore, the PI-BLS approximation satisfies the governing PDE together with all prescribed boundary and initial conditions at every collocation point. Hence, the obtained solution is fully consistent with the imposed physical constraints. \end{proof}


\section{Experiments and Discussions}\label{SEC_experiment}
To evaluate the performance of the proposed PI-BLS, we consider several linear PDEs. First, we provide details of the implementation used in the empirical study, followed by discussions of the test cases.

\subsection{Implementation Details}
All experiments are conducted under a unified computational environment to ensure a fair comparison among the competing methods. The implementation details, test problems, baseline methods, hyperparameter configuration, and metrics are discussed below.
\subsubsection{Setup}\label{sec:setup}
All experiments were implemented in \texttt{Python 3.14.3} using the \texttt{PyTorch 2.11.0} framework with CUDA~12.6 and cuDNN~9.10 acceleration. The experiments were conducted on a workstation equipped with an Intel\textsuperscript{\textregistered} Xeon\textsuperscript{\textregistered} Platinum 8260 CPU (24 cores, 48 logical processors, 2.30~GHz), 256~GB RAM, and an NVIDIA RTX A4500 GPU running the Windows~10 operating system. For all learning models, hyperparameters were selected using a 5-fold cross-validation strategy combined with grid search. Specifically, the training data were partitioned into five folds, and for each candidate hyperparameter configuration, the validation RMSE was computed independently on each fold. The average validation RMSE over the five folds was used as the model selection criterion, and the hyperparameter configuration yielding the lowest average validation RMSE was selected. The final model was subsequently retrained using the selected hyperparameters, and its performance was evaluated on an independent test set using RMSE, MSE, MAE and their corresponding relative error metrics. For a fair comparison, all competing methods are trained and evaluated using identical training and testing points for each benchmark problem.

\subsubsection{Test Problems}

The test cases of PDEs are taken from \cite{berg2018unified, gradientPINN}. In the first example, we provide a detailed construction of the matrices and the associated loss function to help the reader gain a better understanding of the method. Let $\bm{u}(X)$ and $\widetilde{\bm{u}}(X; W_u^*)$ be the exact solution of the PDE and PI-BLS solution, respectively. 

For the one-dimensional steady-state advection equation, the number of interior collocation points was varied as $N_{\mathrm{train}} \in \{5,10,15,20,25,30\}$, while the test set consisted of $N_{\mathrm{test}}=100$ evaluation points. The interior collocation points were generated by randomly sampling points from the computational domain $(0,1)$.

For the two-dimensional Poisson equation, the number of interior collocation points was varied as $N_{\mathrm{train}} \in \{50,60,70,80,90,100\}$, and the solution accuracy was evaluated using $N_{\mathrm{test}}=10,000$ test points. The interior collocation points were randomly sampled within the irregular computational domain using a polar-coordinate parameterization, thereby ensuring an approximately uniform spatial distribution while accurately preserving the domain geometry.

For the diffusion--reaction equation, the number of interior collocation points was also varied as $N_{\mathrm{train}} \in \{50,60,70,80,90,100\}$, with $N_{\mathrm{test}}=10,000$ evaluation points. The interior collocation points were generated by uniformly sampling the spatial variable over $[-\pi,\pi]$ and the temporal variable over $[0,1]$, producing space--time samples at which the governing PDE residual is enforced.

\subsubsection{Baseline Methods}
The proposed PI-BLS framework is evaluated through comprehensive comparisons with the conventional PINN \cite{PINN2019} with two hidden layers. We also take two representative randomized physics-informed learning models, namely PI-ELM \cite{PIELM2020} and B-PI-ELM \cite{PINN2023bayesian}, on all three benchmark problems. All competing methods are trained and evaluated under identical experimental settings to ensure a fair and unbiased comparison.

\subsubsection{Hyperparameter Configuration}

For each benchmark problem, the hyperparameters of all competing methods were selected using the 5-fold cross-validation and grid-search procedure described previously. Since different PDEs impose different physical constraints, the search spaces were adjusted accordingly. In particular, the initial-condition penalty parameter ($\lambda_I$) was introduced only for the time-dependent diffusion--reaction equation, whereas the regularization parameter ($\lambda_{\mathrm{reg}}$) was considered only when applicable.  The hyperparameter search spaces for each benchmark problem are summarized in Tables~\ref{tab:tc1_hyper}, \ref{tab:poisson_hyper}, and \ref{tab:dr_hyper}. 

\begin{table}[!t]
\centering
\caption{Hyperparameter search space for the one-dimensional steady-state advection equation problem.}
\label{tab:tc1_hyper}
\renewcommand{\arraystretch}{1.2}
\resizebox{12cm}{!}{
\begin{tabular}{ll}
\hline
\textbf{Method} & \textbf{Search Space} \\
\hline
PI-BLS &
$N_c\!\in\!\{5,13,21\}$,
$p\!\in\!\{10,30,50\}$,
$q\!\in\!\{25,65,105\}$,
$\lambda_b\!\in\!\{10^{-4},10^{-2},1\}$,
$\sigma\!\in\!\{\text{ReLU},\text{Sigmoid},\text{Tanh}\}$ \\

PINN &
$N_h\!\in\!\{3,23,\ldots,203\}$, 
$L_h=2$,
$\lambda_b\!\in\!\{10^{-4},10^{-2},1\}$,
$\sigma\!\in\!\{\text{ReLU},\text{Sigmoid},\text{Tanh}\}$ \\

PI-ELM &
$N_h\!\in\!\{3,23,\ldots,203\}$,
$\lambda_b\!\in\!\{10^{-4},10^{-2},1\}$,
$\sigma\!\in\!\{\text{ReLU},\text{Sigmoid},\text{Tanh}\}$ \\

B-PI-ELM &
$N_h\!\in\!\{3,23,\ldots,203\}$,
$\lambda_b\!\in\!\{10^{-4},10^{-2},1\}$,
$\sigma\!\in\!\{\text{ReLU},\text{Sigmoid},\text{Tanh}\}$ \\
\hline
\end{tabular}}
\end{table}

\begin{table}[!t]
\centering
\caption{Hyperparameter search space for the two-dimensional Poisson equation.}
\label{tab:poisson_hyper}
\renewcommand{\arraystretch}{1.2}
\resizebox{12cm}{!}{
\begin{tabular}{ll}
\hline
\textbf{Method} & \textbf{Search Space} \\
\hline
PI-BLS &
$N_c\!\in\!\{5,13,21\}$,
$p\!\in\!\{10,30,50\}$,
$q\!\in\!\{25,65,105\}$,
$\lambda_b\!\in\!\{10^{-4},10^{-2},1\}$,
$\lambda_{\rm reg}\!\in\!\{10^{-4},10^{-2},1\}$,
$\sigma\!\in\!\{\text{ReLU},\text{Sigmoid},\text{Tanh}\}$ \\

PINN &
$N_h\!\in\!\{3,23,\ldots,203\}$,
$L_h=2$,
$\lambda_b\!\in\!\{10^{-4},10^{-2},1\}$,
$\sigma\!\in\!\{\text{ReLU},\text{Sigmoid},\text{Tanh}\}$ \\

PI-ELM &
$N_h\!\in\!\{3,23,\ldots,203\}$,
$\lambda_b\!\in\!\{10^{-4},10^{-2},1\}$,
$\sigma\!\in\!\{\text{ReLU},\text{Sigmoid},\text{Tanh}\}$ \\

B-PI-ELM &
$N_h\!\in\!\{3,23,\ldots,203\}$,
$\lambda_b\!\in\!\{10^{-4},10^{-2},1\}$,
$\sigma\!\in\!\{\text{ReLU},\text{Sigmoid},\text{Tanh}\}$ \\
\hline
\end{tabular}}
\end{table}

\begin{table}[!t]
\centering
\caption{Hyperparameter search space for the diffusion--reaction equation.}
\label{tab:dr_hyper}
\renewcommand{\arraystretch}{1.2}
\resizebox{12cm}{!}{
\begin{tabular}{ll}
\hline
\textbf{Method} & \textbf{Search Space} \\
\hline
PI-BLS &
$N_c\!\in\!\{5,13,21\}$,
$p\!\in\!\{10,30,50\}$,
$q\!\in\!\{25,65,105\}$,
$\lambda_b\!\in\!\{10^{-4},10^{-2},1\}$,
$\lambda_I\!\in\!\{10^{-4},10^{-2},1\}$,
$\sigma\!\in\!\{\text{ReLU},\text{Sigmoid},\text{Tanh}\}$ \\

PINN &
$N_h\!\in\!\{3,23,\ldots,203\}$,
$L_h=2$,
$\lambda_b\!\in\!\{10^{-4},10^{-2},1\}$,
$\lambda_I\!\in\!\{10^{-4},10^{-2},1\}$,
$\sigma\!\in\!\{\text{ReLU},\text{Sigmoid},\text{Tanh}\}$ \\

PI-ELM &
$N_h\!\in\!\{3,23,\ldots,203\}$,
$\lambda_b\!\in\!\{10^{-4},10^{-2},1\}$,
$\lambda_I\!\in\!\{10^{-4},10^{-2},1\}$,
$\sigma\!\in\!\{\text{ReLU},\text{Sigmoid},\text{Tanh}\}$ \\

B-PI-ELM &
$N_h\!\in\!\{3,23,\ldots,203\}$,
$\lambda_b\!\in\!\{10^{-4},10^{-2},1\}$,
$\lambda_I\!\in\!\{10^{-4},10^{-2},1\}$,
$\sigma\!\in\!\{\text{ReLU},\text{Sigmoid},\text{Tanh}\}$ \\
\hline
\end{tabular}}
\end{table}

\subsubsection{Metrics} To assess the accuracy of the predicted solutions, we evaluate the methods using the metrics tabulated in Table \ref{tab:metrics}.

\begin{table*}[t]
\centering
\caption{Performance evaluation metrics. Here, $N$ is the total number of test samples, $\bm{u}(X)$ is the exact solution, $\widetilde{\bm{u}}(X)$ is the predicted solution, and $\bar{\bm{u}}=\frac{1}{N}\sum_{i=1}^{N}\bm{u}(X_i)$.}
\label{tab:metrics}
\small
\resizebox{12cm}{!}{
\begin{tabular}{p{4cm}p{6cm}}
\toprule
\textbf{Metric} & \textbf{Expression} \\
\midrule

MSE &
$\displaystyle
\frac{1}{N}\sum_{i=1}^{N}\left(\widetilde{\bm{u}}(X_i)-\bm{u}(X_i)\right)^2$
\\

RMSE &
$\displaystyle
\sqrt{\frac{1}{N}\sum_{i=1}^{N}\left(\widetilde{\bm{u}}(X_i)-\bm{u}(X_i)\right)^2}$
\\

MAE &
$\displaystyle
\frac{1}{N}\sum_{i=1}^{N}\left|\widetilde{\bm{u}}(X_i)-\bm{u}(X_i)\right|$
\\

Relative MSE &
$\displaystyle
\frac{\sum_{i=1}^{N}\left(\widetilde{\bm{u}}-\bm{u}\right)^2}
{\sum_{i=1}^{N}\bm{u}^2}$
\\

Relative RMSE ($L^2$ Error) &
$\displaystyle
\frac{\sqrt{\sum_{i=1}^{N}\left(\widetilde{\bm{u}}-\bm{u}\right)^2}}
{\sqrt{\sum_{i=1}^{N}\bm{u}^2}}$
\\

Relative MAE &
$\displaystyle
\frac{\sum_{i=1}^{N}\left|\widetilde{\bm{u}}-\bm{u}\right|}
{\sum_{i=1}^{N}\left|\bm{u}\right|}$
\\

$L_\infty$ Error &
$\displaystyle
\max_{1\le i\le N}\left|\widetilde{\bm{u}}(X_i)-\bm{u}(X_i)\right|$
\\

Relative $L_\infty$ Error &
$\displaystyle
\frac{\max_{1\le i\le N}\left|\widetilde{\bm{u}}-\bm{u}\right|}
{\max_{1\le i\le N}\left|\bm{u}\right|}$
\\

$R^2$ Score &
$\displaystyle
1-\frac{\sum_{i=1}^{N}\left(\widetilde{\bm{u}}-\bm{u}\right)^2}
{\sum_{i=1}^{N}\left(\bm{u}-\bar{\bm{u}}\right)^2}$
\\

\bottomrule
\end{tabular}}
\end{table*}

  \subsection{Discussion: PI-BLS for Steady Advection Equation}
The one-dimensional advection equation is expressed as follows:
\begin{align}
    \begin{array}{rc}
        & \bm{u}_x=R(x), ~0<x\leq 1. 
    \end{array}
\end{align}
  The expressions for 
$R$ and the Dirichlet boundary conditions are derived based on the following exact solution:
  \begin{align}
      &{\bm u}(x)=\sin(2\pi x)\cos(4\pi x)+1.
  \end{align}
  The differential operator $\mathcal{N}_{x}[\,\cdot \,]$ for these problems is $\frac{\partial}{\partial x}.$
  The boundary operator is the identity for all the problems. 

  The expressions for the error function due to PDE residual  
 is given by:
 \begin{align}\label{erro_PDE_TC1}
&\mathcal{J}_{u}= \widetilde{\bm{u}}_x(x_i)-{R}(x_i)~~\text{on}~~\{x_{i, u}\}_{i=1}^{N_c}.
  \end{align}
  The error function for the boundary condition is given by 
  \begin{align*}
      \label{error_BC2}
&\mathcal{J}_{b}=\widetilde{\bm{u}}(x)-B(x)~~ \text{on}~~\{x_{i,b}\}_{i=1}^{N_b}.
  \end{align*}
 In this example, we directly calculated the derivatives of the approximate functions for the listed PDEs.  
The coefficient matrices are given by:
\[H_{\mathrm{col}} 
= \frac{\partial \mathcal{G}}{\partial x}(X_{\mathrm{col}})
= \bmatrix{\textbf{F}_{x}(X_{\text{col}}) & \textbf{E}_{x}(X_{\text{col}})} ~~\text{and}~~H_b=\bmatrix{E(X_b) & F(X_{b})}. \]
\begin{table}[t!]
\centering
\caption{Comparison of performance metrics for PI-BLS, PINN, PI-ELM, and B-PI-ELM methods for the advection equation with  30 collocation points.}
\label{Tab:advection}
\resizebox{12cm}{!}{
\begin{tabular}{lcccccccc}
\toprule
& \multicolumn{2}{c}{PI-BLS} & \multicolumn{2}{c}{PINN} & \multicolumn{2}{c}{PI-ELM} & \multicolumn{2}{c}{B-PI-ELM} \\
\cmidrule(lr){2-3} \cmidrule(lr){4-5} \cmidrule(lr){6-7} \cmidrule(lr){8-9}
Parameter & Train & Test & Train & Test & Train & Test & Train & Test \\
\midrule
MSE     & 4.5149e-07 & 4.9807e-07 & 2.2248e-05 & 2.1427e-05 & 3.5523e-01 & 4.2637e-01 &    1.1356e+00&1.1407e+0
  \\
RMSE    & 6.7193e-04 & 7.0573e-04 &  4.7167e-03 & 4.6290e-03 & 5.9601e-01 & 6.5297e-01 &    1.0656e+00& 
 1.0680e+00
 \\
MAE     & 5.5396e-04 & 5.7593e-04 & 4.5665e-03 &  4.4032e-03 & 4.3809e-01 & 5.1484e-01 &   9.5790e-01 & 9.3827e-01\\
 Relative MSE  & 3.9758e-07 & 4.3665e-07 & 1.9591e-05 & 1.8785e-05 & 3.1281e-01 & 3.7379e-01 &  1.0000e+00 & 1.0000e+00  \\
Relative RMSE & 6.3054e-04 & 6.6079e-04 &  4.4262e-03 & 4.3342e-03 & 5.5930e-01 & 6.1138e-01 &  1.0000e+00 & 1.0000e+00   \\
Relative MAE  & 5.7830e-04 & 6.1382e-04 & 4.7672e-03 & 4.6928e-03 & 4.5734e-01 & 5.4870e-01 &  1.0000e+00 & 1.0000e+00  \\
\bottomrule
\end{tabular}
}
\end{table}
Table~\ref{Tab:advection} presents a comparative performance analysis of the proposed PI-BLS against PINN, PI-ELM, and B-PI-ELM for solving the advection equation using $30$ training collocation points. As observed, PI-BLS consistently outperforms the competing methods across all evaluation metrics. It achieves the lowest training and testing values of MSE, RMSE, MAE, and their corresponding relative error measures, demonstrating its superior approximation accuracy and generalization capability. On the testing dataset, PI-BLS attains an MSE of $4.9807\times10^{-7}$, which is approximately $43$ times smaller than that of PINN and several orders of magnitude lower than those of PI-ELM and B-PI-ELM. Similar improvements are observed for RMSE, MAE, and the corresponding relative error metrics. Moreover, the relative MSE, relative RMSE, and relative MAE obtained by PI-BLS remain in the range of $10^{-7}$--$10^{-4}$, whereas PI-ELM and B-PI-ELM exhibit significantly larger normalized errors, with B-PI-ELM approaching unity. These results demonstrate the superior numerical accuracy, stability, and robustness of the proposed PI-BLS framework.
\begin{figure}
    \centering
    \begin{minipage}{0.32\textwidth}
        \centering
        \includegraphics[width=\linewidth]{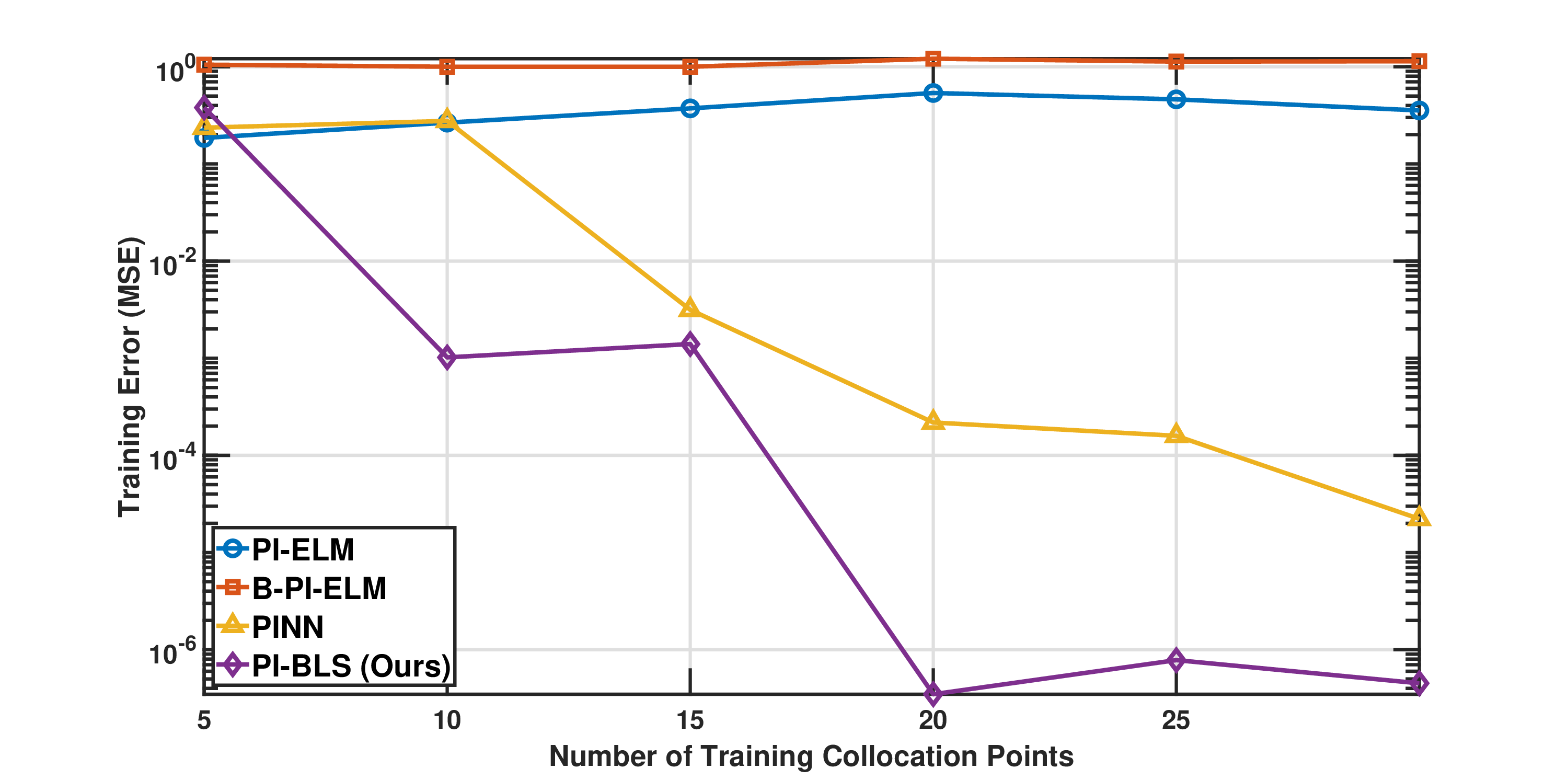}
        \caption*{MSE error in training}
    \end{minipage}
    \begin{minipage}{0.32\textwidth}
        \centering
        \includegraphics[width=\linewidth]{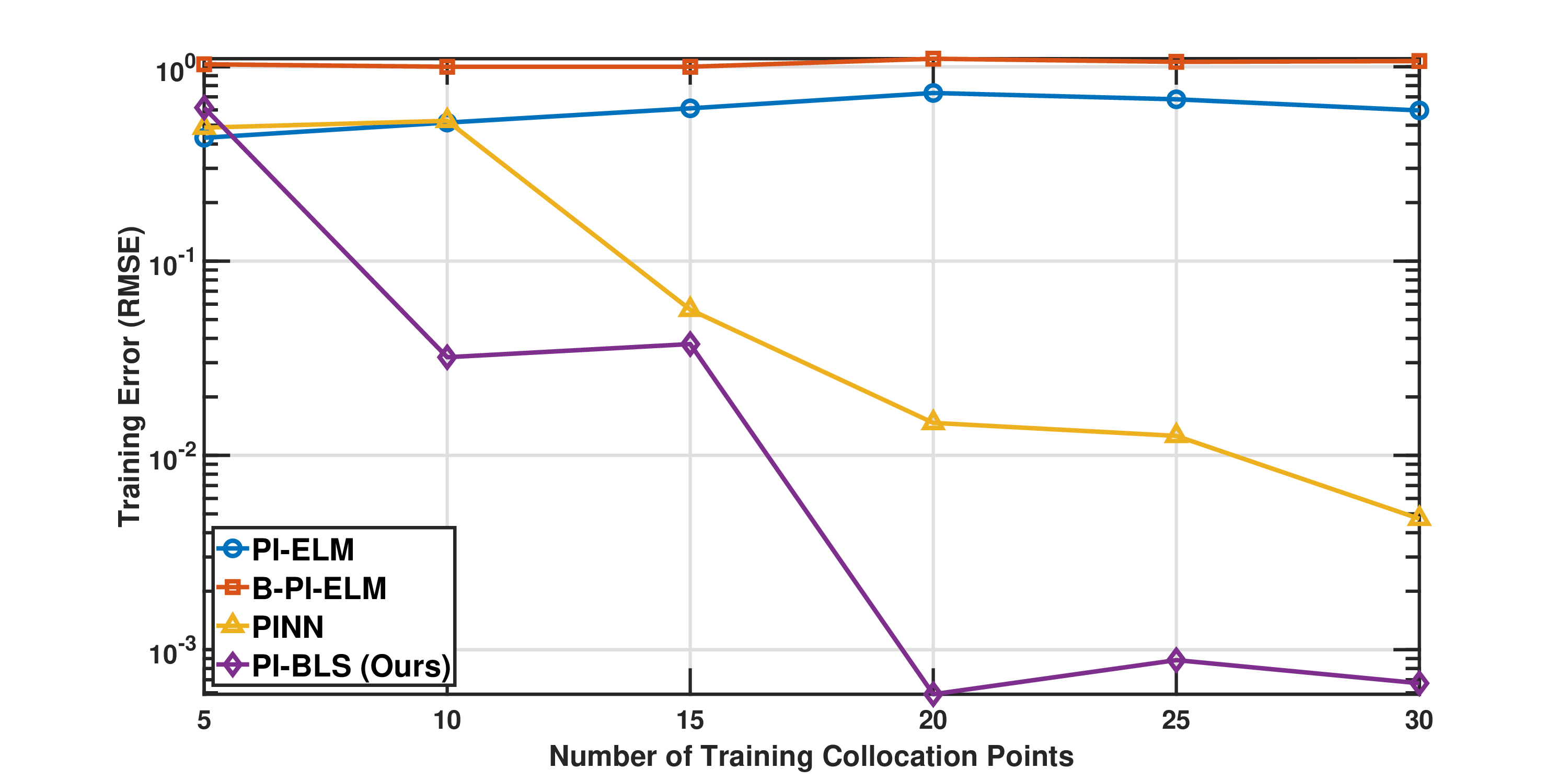}
        \caption*{RMSE error in training}
    \end{minipage}
    \begin{minipage}{0.32\textwidth}
        \centering
        \includegraphics[width=\linewidth]{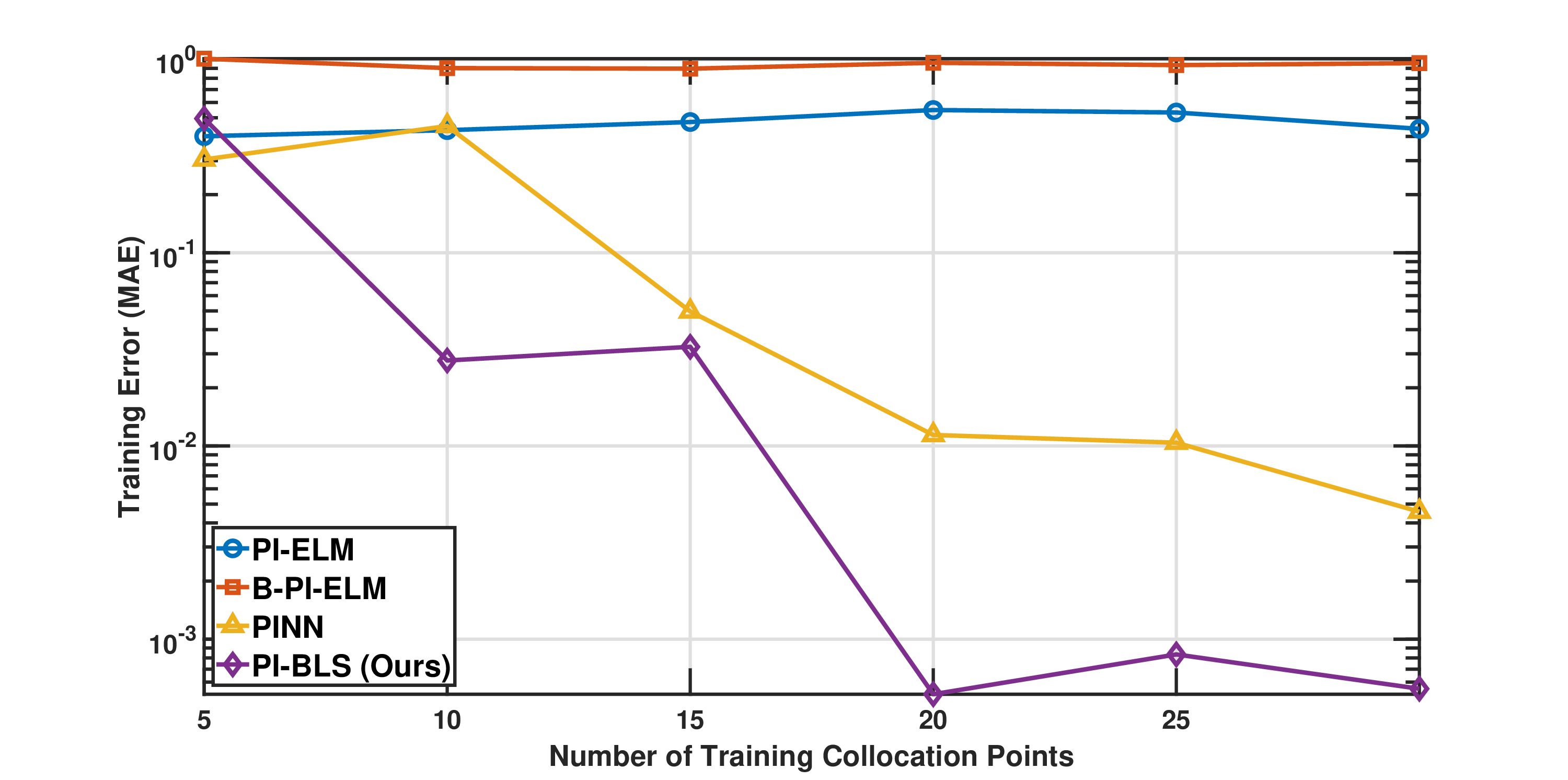}
        \caption*{MAE error in training}
    \end{minipage}

    \vspace{0.5cm} 

    \begin{minipage}{0.32\textwidth}
        \centering
        \includegraphics[width=\linewidth]{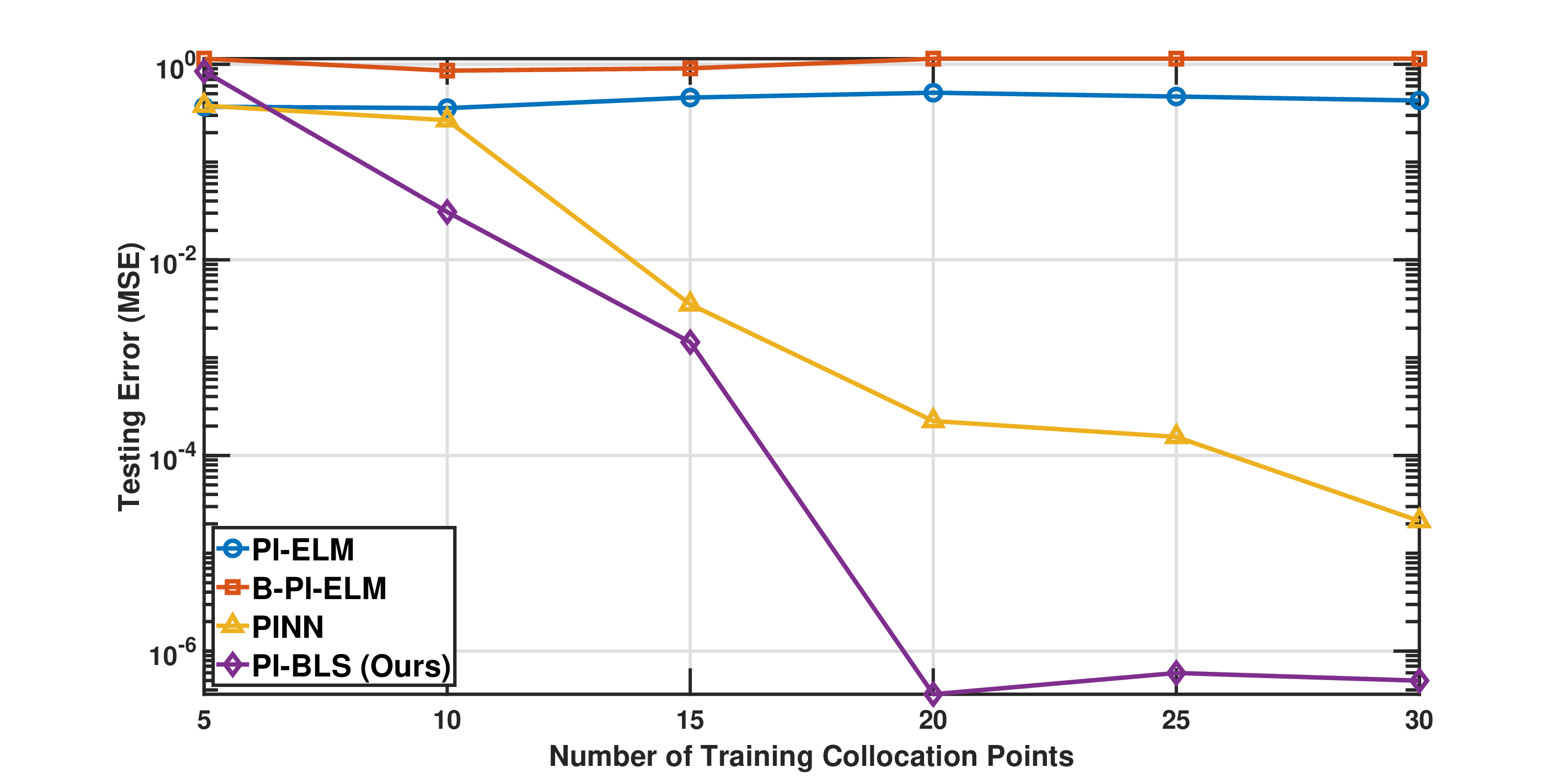}
        \caption*{MSE error in testing}
    \end{minipage}
    \begin{minipage}{0.32\textwidth}
        \centering
        \includegraphics[width=\linewidth]{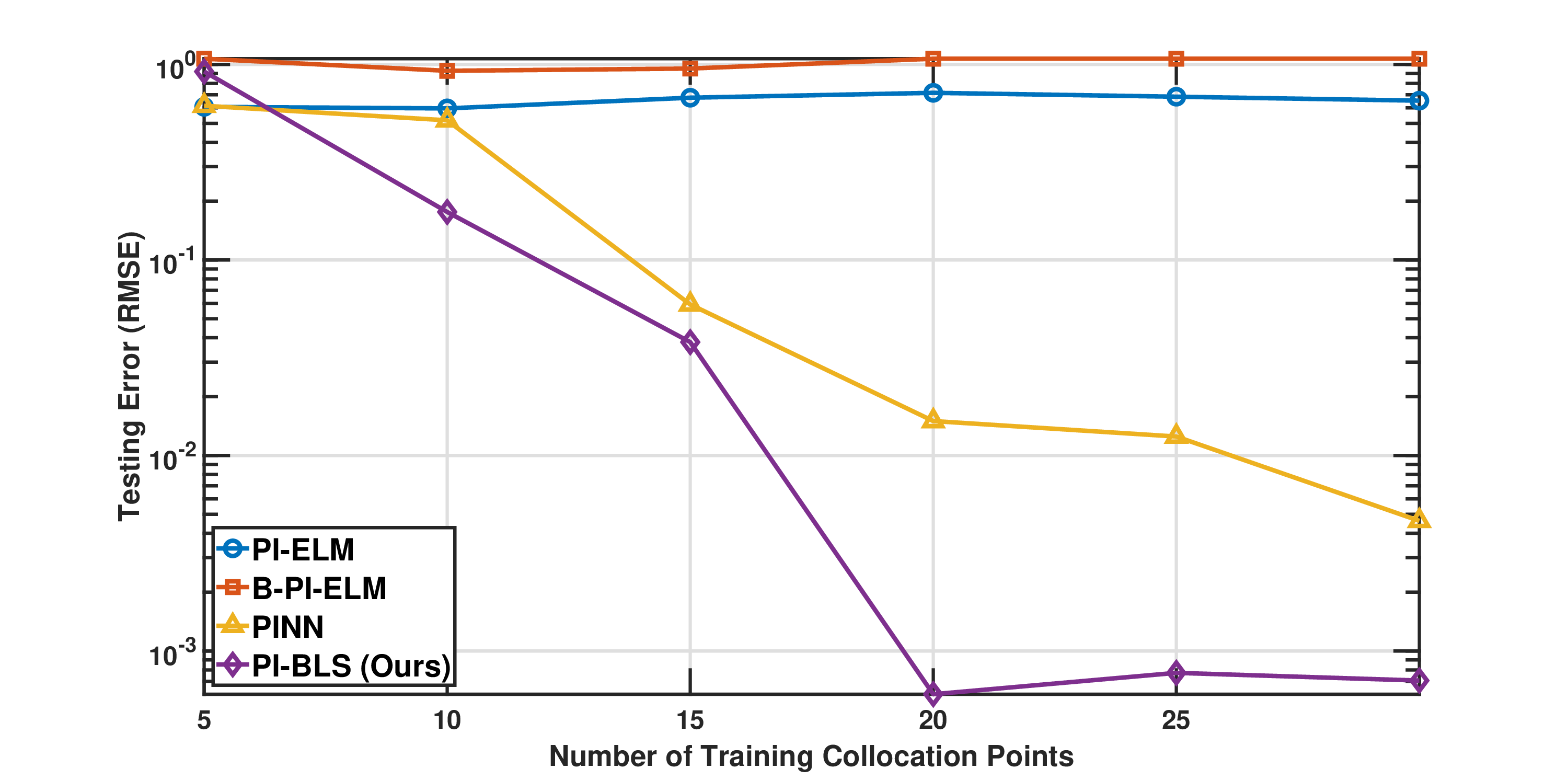}
        \caption*{RMSE error in testing}
    \end{minipage}
    \begin{minipage}{0.32\textwidth}
        \centering
        \includegraphics[width=\linewidth]{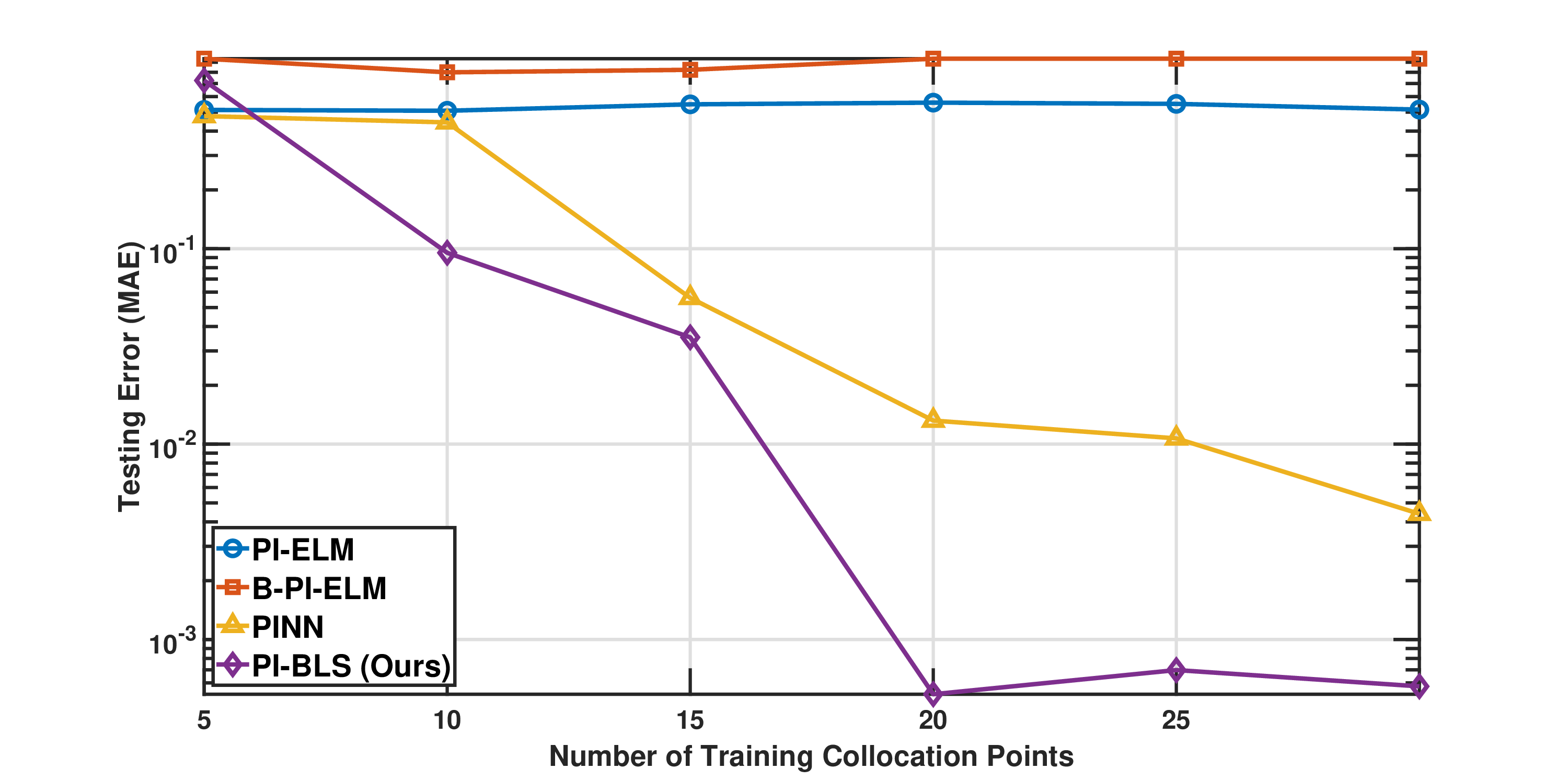}
        \caption*{MAE error in testing}
    \end{minipage}
    \caption{Comparison of various error metrics for the proposed PI-BLS, PINN, PI-ELM, and B-PI-ELM methods across different numbers of training points for the advection equation.}
    \label{fig:advection_error_metrics}
\end{figure}

\begin{figure}
    \centering
    \begin{minipage}{0.75\textwidth}
        \centering
        \includegraphics[width=\linewidth]{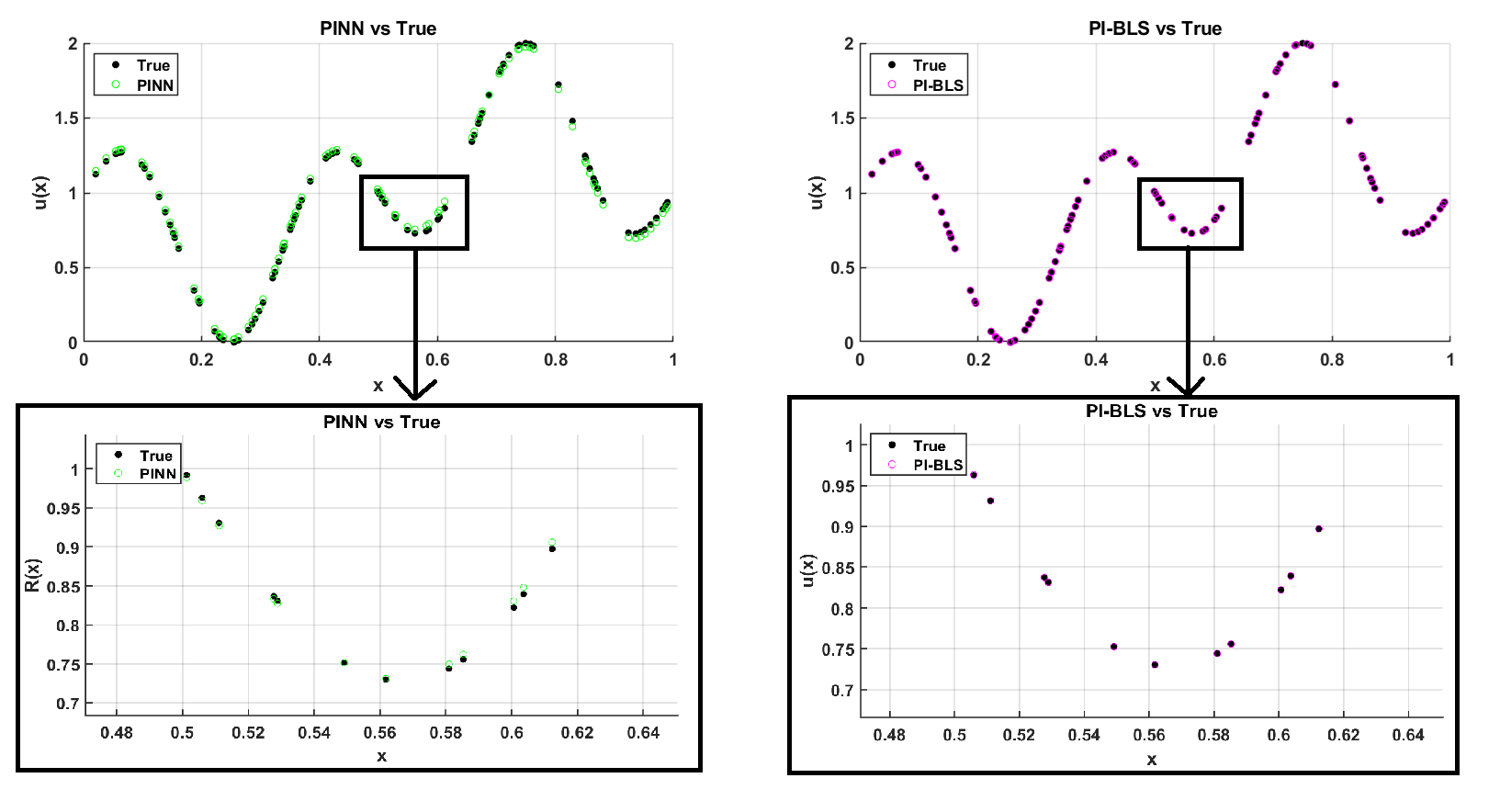}
        \subcaption{Approximate vs true solution}
    \end{minipage}
    \begin{minipage}{0.75\textwidth}
        \centering
        \includegraphics[width=\linewidth]{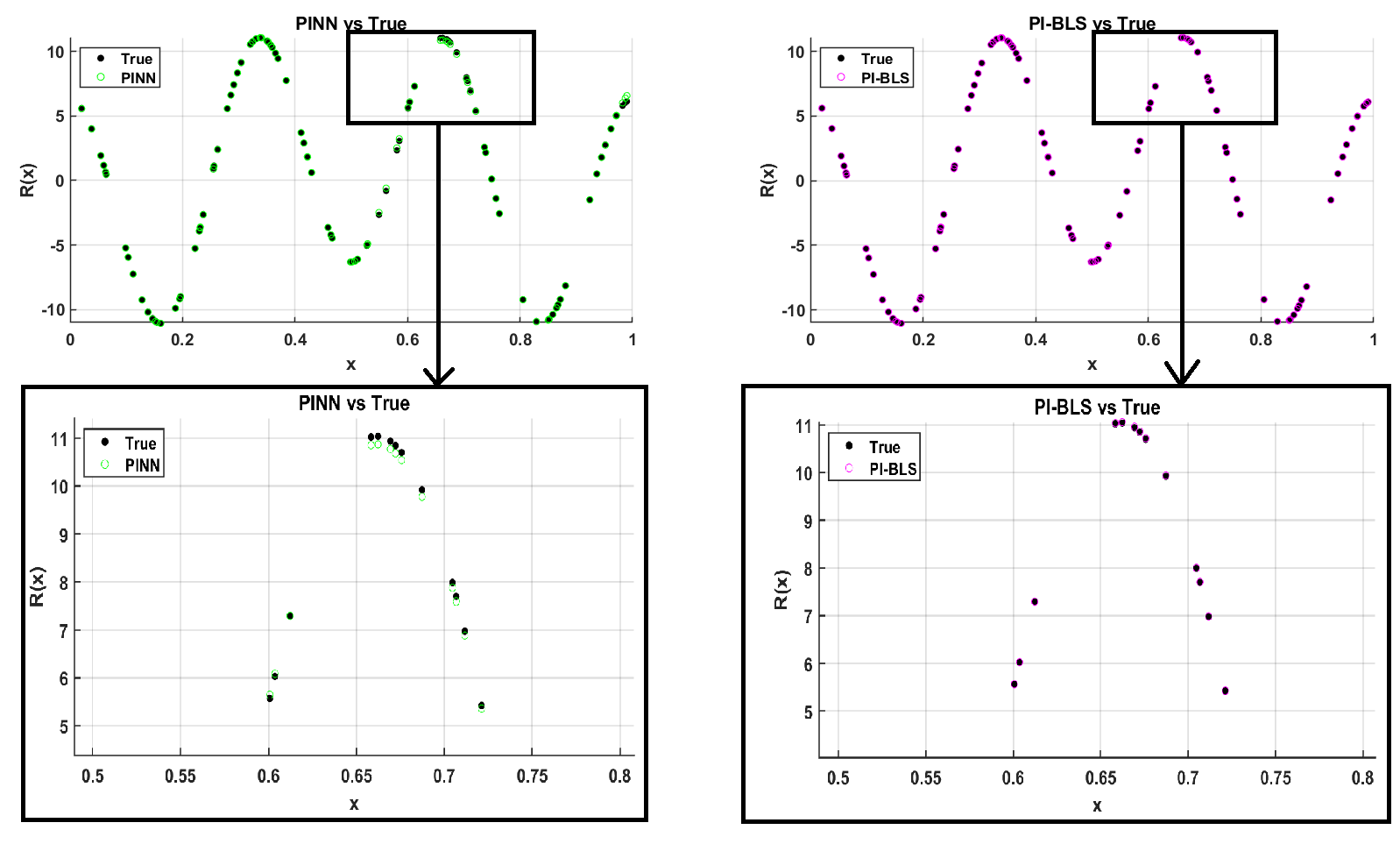}
        \subcaption{Approximate vs exact $R(x)$}
    \end{minipage}
    \begin{minipage}{0.5\textwidth}
        \centering
        \includegraphics[width=\linewidth]{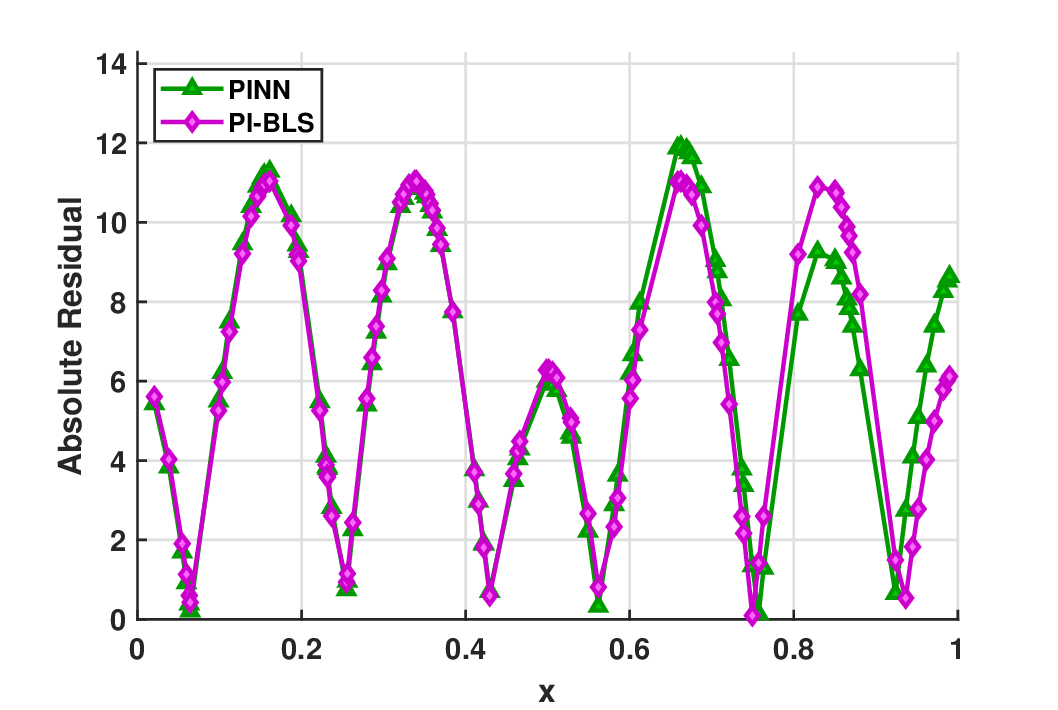}
        \subcaption{Absolute residual varying $x$}
    \end{minipage}
    \caption{Comparison of the approximate solution vs true solution for the advection equation.}
    \label{fig:advection_comparison}
\end{figure}
%
Figure~\ref{fig:advection_error_metrics} illustrates the variation of the training and testing MSE, RMSE, and MAE with different numbers of training collocation points. The proposed PI-BLS consistently yields the smallest errors for all metrics, with both training and testing errors decreasing rapidly as the number of collocation points increases and eventually converging to very small values. In contrast, PINN, PI-ELM, and B-PI-ELM exhibit comparatively larger errors throughout the experiments, indicating slower convergence and reduced prediction accuracy. These results further confirm the effectiveness and scalability of PI-BLS for physics-informed learning.

Figure~\ref{fig:advection_comparison} provides a qualitative comparison of the approximate solutions obtained by PI-BLS and PINN with the corresponding exact solution. As shown in Fig.~\ref{fig:advection_comparison}(a), both methods successfully capture the overall solution profile; however, the PI-BLS solution exhibits a closer agreement with the exact solution across the entire computational domain. The comparison of the computed and exact right-hand side $R(x)$ in Fig.~\ref{fig:advection_comparison}(b) further demonstrates that PI-BLS more accurately reproduces the governing physical relationship than PINN. This observation is corroborated by the absolute residual distributions shown in Fig.~\ref{fig:advection_comparison}(c), where PI-BLS consistently produces smaller residuals, indicating a more accurate satisfaction of the governing equation and enhanced numerical reliability.
%



   \subsection{Discussion: PI-BLS for Two-Dimensional Steady Poisson Equation}
We  consider a $2\mathrm{D}$ Poisson problem \cite{gradientPINN} \[-(\bm{u}_{xx}+ \bm{u}_{yy})=R(x,y), ~~(x,y)\in [0,1]^2,\]
where $R$ is obtained from the exact solution 
\[\bm{u}(x,y)= 2^{4a} x^{a}(1-x)^ay^a(1-y^a).\]
We consider the zero Dirichlet boundary conditions on the boundary.
    %

The PI-BLS loss functions are given as follows:
\begin{align}\label{erro_PDE_2D}
\text{PDE loss:}&~~~\mathcal{J}_{u}= \widetilde{\bm{u}}_{xx}(x, y)+\widetilde{\bm{u}}_{yy}(x,y)-{R}(x, y)~~\text{on}~~\{x_{i, u}, y_{i, u}\}_{i=1}^{N_c}. \\
\text{Boundary loss:}&~~~ \mathcal{J}_{b}=\widetilde{\bm{u}}(x,y)-B(x,y)~~ \text{on}~~\{x_{i,b}, y_{i,b}\}_{i=1}^{N_b},
  \end{align}
  where the input $X=[x,\, y].$ 
The formulation for the coefficient matrices is given by:
\begin{align*}
    H_{\mathrm{col}} &= \frac{\partial^2 \mathcal{G}}{\partial x^2}(X_{\mathrm{col}}) + \frac{\partial^2 \mathcal{G}}{\partial y^2}(X_{\mathrm{col}})\\
    &= \bmatrix{-(\textbf{F}_{xx,\text{col}} + \textbf{F}_{yy,\text{col}}) & -( \textbf{E}_{xx,\text{col}} +  \textbf{E}_{yy,\text{col}}) },\\
     H_b&=\bmatrix{E(X_b) & F(X_{b})}. 
\end{align*}

\begin{table}[!t]
\centering
\caption{Comparison of performance metrics for PI-BLS and PINN for the Poisson equation in 2D problem with 100 training collocation points.}
\label{tab:poisson2d_results}
\renewcommand{\arraystretch}{1.15}
\resizebox{9cm}{!}{
\begin{tabular}{lcccc}
\toprule
& \multicolumn{2}{c}{PI-BLS} & \multicolumn{2}{c}{PINN} \\
\cmidrule(r){2-3}\cmidrule(l){4-5}
\textbf{Parameter} & \textbf{Train} & \textbf{Test} & \textbf{Train} & \textbf{Test} \\
\midrule

MSE &
4.5950e-06 & 5.8687e-06 &
1.5526e-04 & 1.5591e-04 \\

RMSE &
2.1436e-03 & 2.4225e-03 &
1.2460e-02 & 1.2486e-02 \\

MAE &
1.6879e-03 & 1.8196e-03 &
9.0178e-03 & 8.5803e-03 \\


$\ell_{\infty}$  &
5.7685e-03 & 1.6726e-02 &
4.5706e-02 & 8.3972e-02 \\

Relative MSE &
3.3228e-06 & 3.9407e-06 &
1.1228e-04 & 1.0469e-04 \\

Relative RMSE &
1.8229e-03 & 1.9851e-03 &
1.0596e-02 & 1.0232e-02 \\

Relative MAE &
1.4842e-03 & 1.5411e-03 &
7.9297e-03 & 7.2672e-03 \\

Relative $\ell_{\infty}$  &
3.8460e-03 & 1.1151e-02 &
3.0473e-02 & 5.5982e-02 \\

$R^2$ Score &
9.9995e-01 & 9.9994e-01 &
9.9827e-01 & 9.9836e-01 \\

\bottomrule
\end{tabular}}
\end{table}

 Table~\ref{tab:poisson2d_results} compares the performance of the proposed PI-BLS and the conventional PINN for the two-dimensional Poisson equation. In the experiment, the number of training collocation points is set to $N_{\text{train}}=100$, while the testing dataset consists of $N_{\text{test}}=10,000$ collocation points. Further, in this example, we consider the error matrices $\ell_\infty$ Error, Relative $\ell_\infty$ Error, and $R^2$ Score. The results show that PI-BLS consistently outperforms PINN across all evaluation metrics. 
Specifically, PI-BLS achieves substantially lower MSE, RMSE, MAE,  and $\ell^\infty$ errors than PINN, indicating a more accurate approximation of the exact solution. On the testing dataset, the MSE is reduced from $1.5591\times10^{-4}$ to $5.8687\times10^{-6}$, representing more than one order of magnitude improvement. Likewise, the RMSE, MAE,  and $\ell^\infty$ errors are all approximately five times smaller than those obtained by PINN. Similar improvements are observed for the relative error metrics, while the $R^2$ score of PI-BLS ($0.99994$) remains closer to the ideal value of one than that of PINN ($0.99836$). These results demonstrate that PI-BLS provides superior approximation accuracy and generalization capability, even when trained using only $100$ collocation points.

\begin{figure}
    \centering
    \begin{minipage}{0.32\textwidth}
        \centering
        \includegraphics[width=\linewidth]{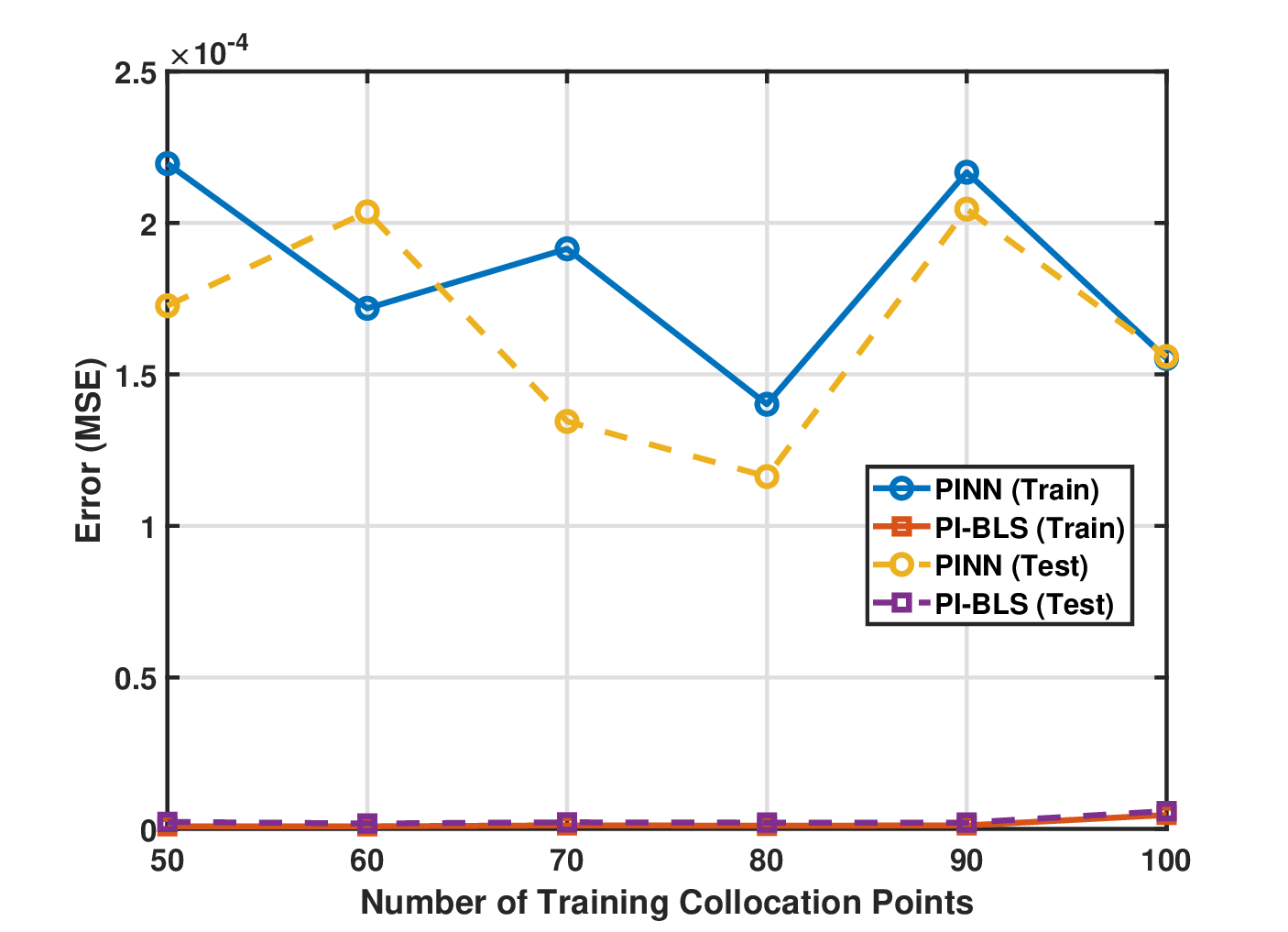}
        \caption*{MSE}
    \end{minipage}
    \begin{minipage}{0.32\textwidth}
        \centering
        \includegraphics[width=\linewidth]{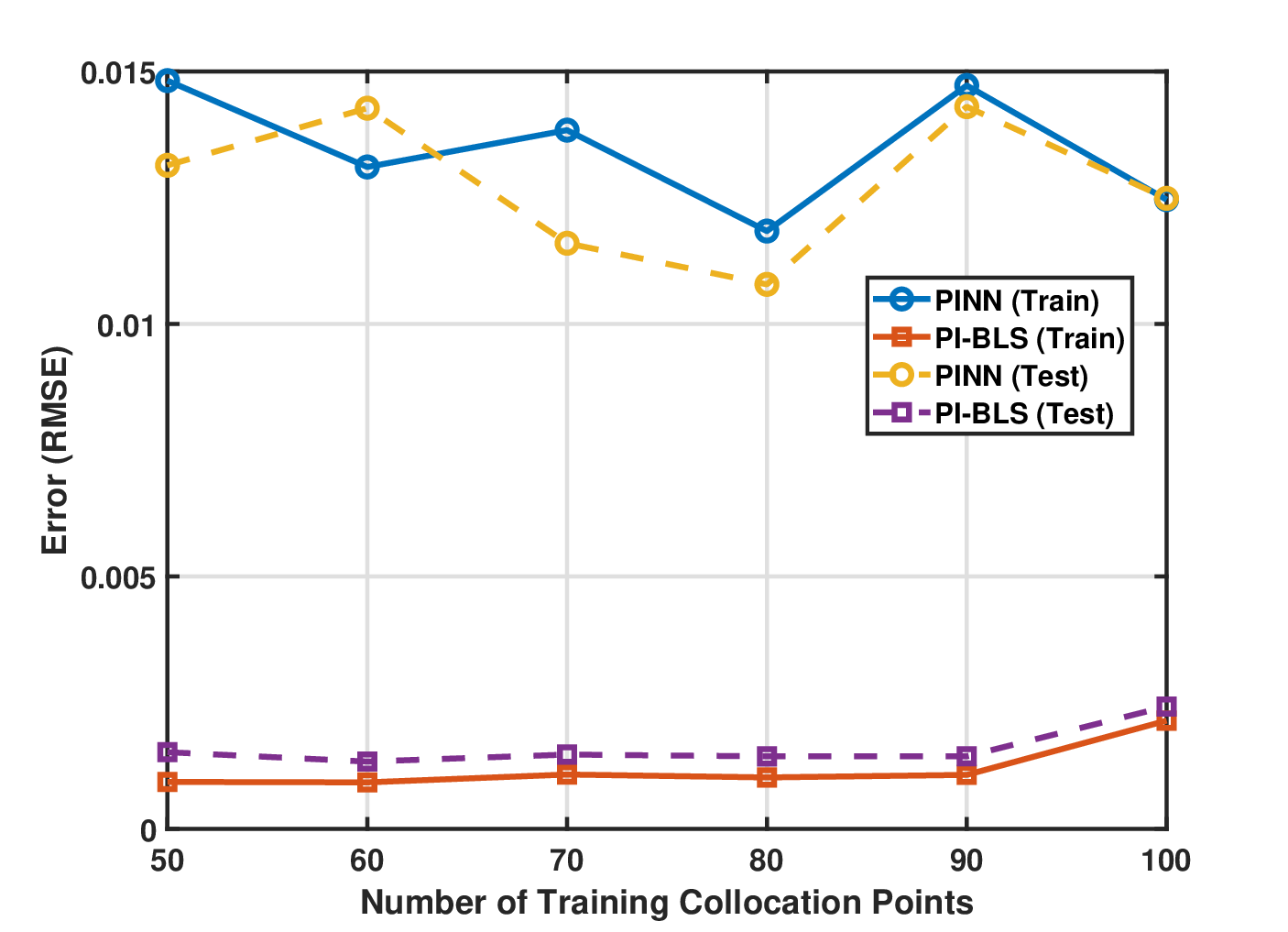}
        \caption*{RMSE}
    \end{minipage}
    \begin{minipage}{0.32\textwidth}
        \centering
        \includegraphics[width=\linewidth]{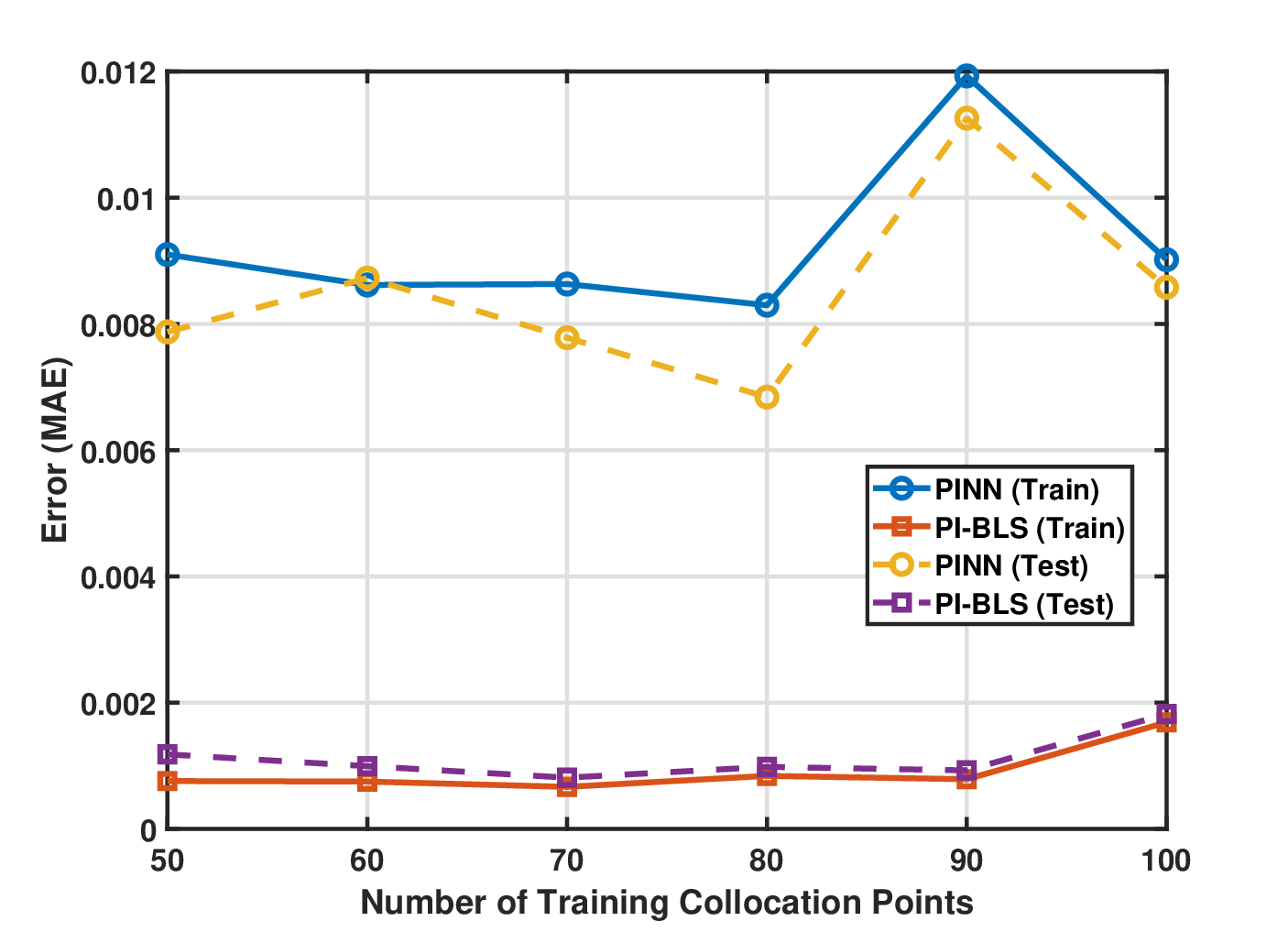}
        \caption*{MAE}
    \end{minipage}
   
    \caption{Comparison of various error metrics for the proposed PI-BLS and PINN methods across different numbers of training points for the Poisson 2D equation.}
    \label{fig:poisson2d}
\end{figure}

Figure~\ref{fig:poisson2d} further illustrates the variation of the MSE, RMSE, and MAE as the number of training collocation points increases from $50$ to $100$. Across all training sizes, PI-BLS consistently maintains significantly lower training and testing errors than PINN. Moreover, the training and testing curves of PI-BLS remain close to each other with relatively small fluctuations, indicating stable convergence, strong generalization, and robustness with respect to the number of training collocation points. These results confirm that the proposed PI-BLS is a more accurate and reliable physics-informed solver than the conventional PINN for the two-dimensional Poisson equation.

  \subsection{Discussion: Unsteady Case :- Diffusion–Reaction Equation }
   We consider the following one-dimensional unsteady diffusion--reaction problem described by the PDE taken as in  \cite{gradientPINN}:
    \begin{align}
        &\frac{\partial\bm{u}}{\partial t}= \mu \frac{\partial^2\bm{u}}{\partial x^2} + R(x,t), ~(x,t)\in [-\pi, \pi]\times [0, 1],
    \end{align}
where $\mu=1$ and \[R(x, t) = e^{-t} \left(\frac{3}{2} \sin (2x) + \frac{8}{3} \sin (3x) \frac{15}{4} \sin (4x) + \frac{63}{8} \sin (8x) \right).\] 
The initial and boundary conditions are given by 
\begin{align*}
    &\bm{u}(x, 0) = \sum_{i=1}^4 \frac{\sin ( i x)}{i} + \frac{\sin ( 8 x)}{8}\\
    & \bm{u}(-\pi, t) =\bm{u}(\pi, t)= 0.
\end{align*}
The exact solution of the PDE is given by \[\bm{u}(x, t) = e^{-t}\left(\sum_{i=1}^4 \frac{\sin ( i x)}{i} + \frac{\sin ( 8 x)}{8} \right).\]
The PI-BLS equations for these functions are given as follows:
\begin{align}\label{erro_PDE_DR}
\text{PDE loss:}&~~~\mathcal{J}_{u}= \widetilde{\bm{u}}_{t}(x, t)- \mu \,\widetilde{\bm{u}}_{x}(x,t)-{R}(x, t)~~\text{on}~~\{x_{i, u}, t_{i, u}\}_{i=1}^{N_c}. \\
\text{Boundary loss:}&~~~ \mathcal{J}_{b}=\widetilde{\bm{u}}(x,t)-B(x,t)~~ \text{on}~~\{x_{i,b}, t_{i,b}\}_{i=1}^{N_b}.\\
\text{Initial condition loss:}&~~~ \mathcal{J}_{I}=\widetilde{\bm{u}}(x,0)-\mathcal{I}(t)~~ \text{on}~~\{x_{i,I}\}_{i=1}^{N_I}.
  \end{align}
The formula for the coefficient matrix is given by:
\[H_{\mathrm{col}} = \frac{\partial\mathcal{G}}{\partial t} (X_{\mathrm{col}}) -\frac{\partial \mathcal{G}}{\partial x} (X_{\mathrm{col}})= \bmatrix{\textbf{F}_{t,\text{col}} -\mu \textbf{F}_{x,\text{col}} &  \textbf{E}_{t,\text{col}} -\mu \textbf{E}_{x,\text{col}}}. \]

\begin{table}[!t]
\centering
\caption{Comparison of performance metrics for PI-BLS, PINN, PI-ELM, and B-PI-ELM for diffusion-reaction for 100 training collocation points.}
\label{tab:diffusion_reaction}
\resizebox{12cm}{!}{
\begin{tabular}{lcccccccc}
\toprule
&
\multicolumn{2}{c}{PI-BLS} &
\multicolumn{2}{c}{PINN} &
\multicolumn{2}{c}{PI-ELM} &
\multicolumn{2}{c}{B-PI-ELM} \\
\cmidrule(lr){2-3}
\cmidrule(lr){4-5}
\cmidrule(lr){6-7}
\cmidrule(lr){8-9}
Parameter &
Train & Test &
Train & Test &
Train & Test &
Train & Test \\
\midrule
MSE &
8.2170e-02 & 8.7424e-02 &
2.3443e-01 & 3.1215e-01 &
1.8193e-01 & 2.2393e-01 &
1.8029e-01 & 2.5049e-01 \\

RMSE &
2.8665e-01 & 2.9568e-01 &
4.8418e-01 & 5.5870e-01 &
4.2653e-01 & 4.7321e-01 &
4.2460e-01 & 5.0049e-01 \\

MAE &
2.4843e-01 & 2.5334e-01 &
4.0089e-01 & 4.5368e-01 &
3.3576e-01 & 3.6853e-01 &
3.4144e-01 & 3.9102e-01 \\

Relative MSE &
3.7369e-01 & 2.8537e-01 &
1.0661e+00 & 1.0189e+00 &
7.1826e-01 & 7.3093e-01 &
8.1991e-01 & 8.1763e-01 \\

Relative RMSE &
6.1130e-01 & 5.3420e-01 &
1.0325e+00 & 1.0094e+00 &
8.4750e-01 & 8.5494e-01 &
9.0549e-01 & 9.0423e-01 \\

Relative MAE &
6.3950e-01 & 5.5976e-01 &
1.0320e+00 & 1.0024e+00 &
8.1043e-01 & 8.1427e-01 &
8.7894e-01 & 8.6396e-01 \\
\bottomrule
\end{tabular}}
\end{table}

\begin{figure}
    \centering
    \begin{minipage}{1\textwidth}
        \centering
        \includegraphics[width=\linewidth]{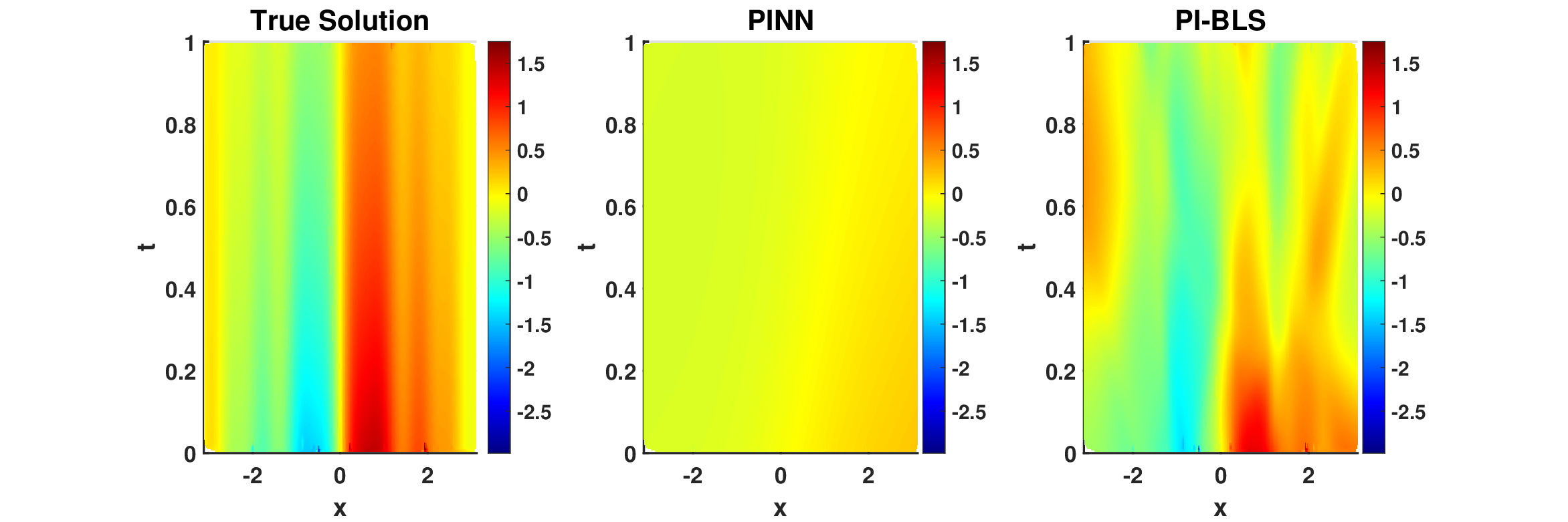}
        \subcaption{Approximate vs true solution}
    \end{minipage}
    \begin{minipage}{0.66\textwidth}
        \centering
        \includegraphics[width=\linewidth]{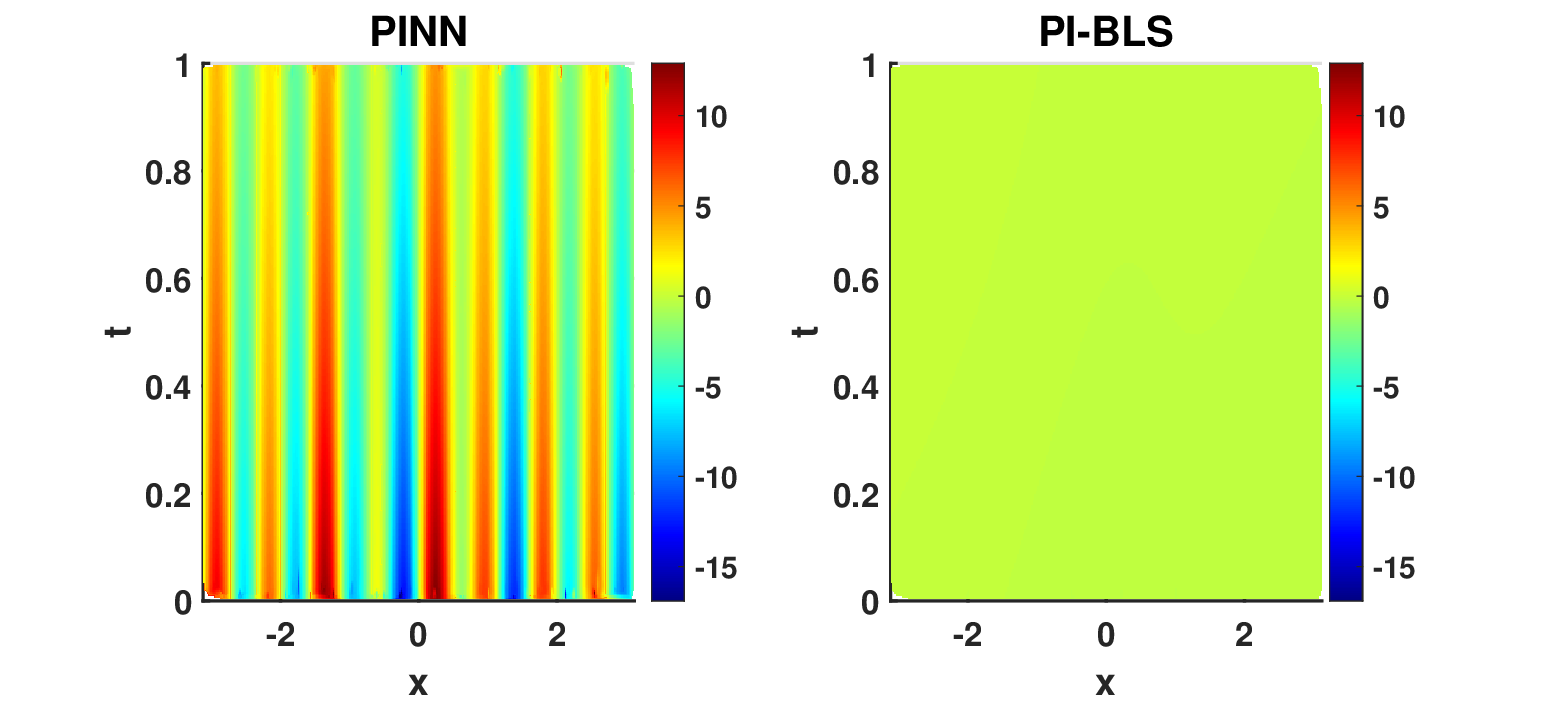}
        \subcaption{Absolute residual varying $x$}
    \end{minipage}
    \caption{Comparison of the approximate solution vs true solution for the diffusion-reaction equation for the number of training points equal 50.}
    \label{fig:diffusion_reaction_error}
\end{figure}

\begin{figure}
    \centering
    \begin{minipage}{0.32\textwidth}
        \centering
        \includegraphics[width=\linewidth]{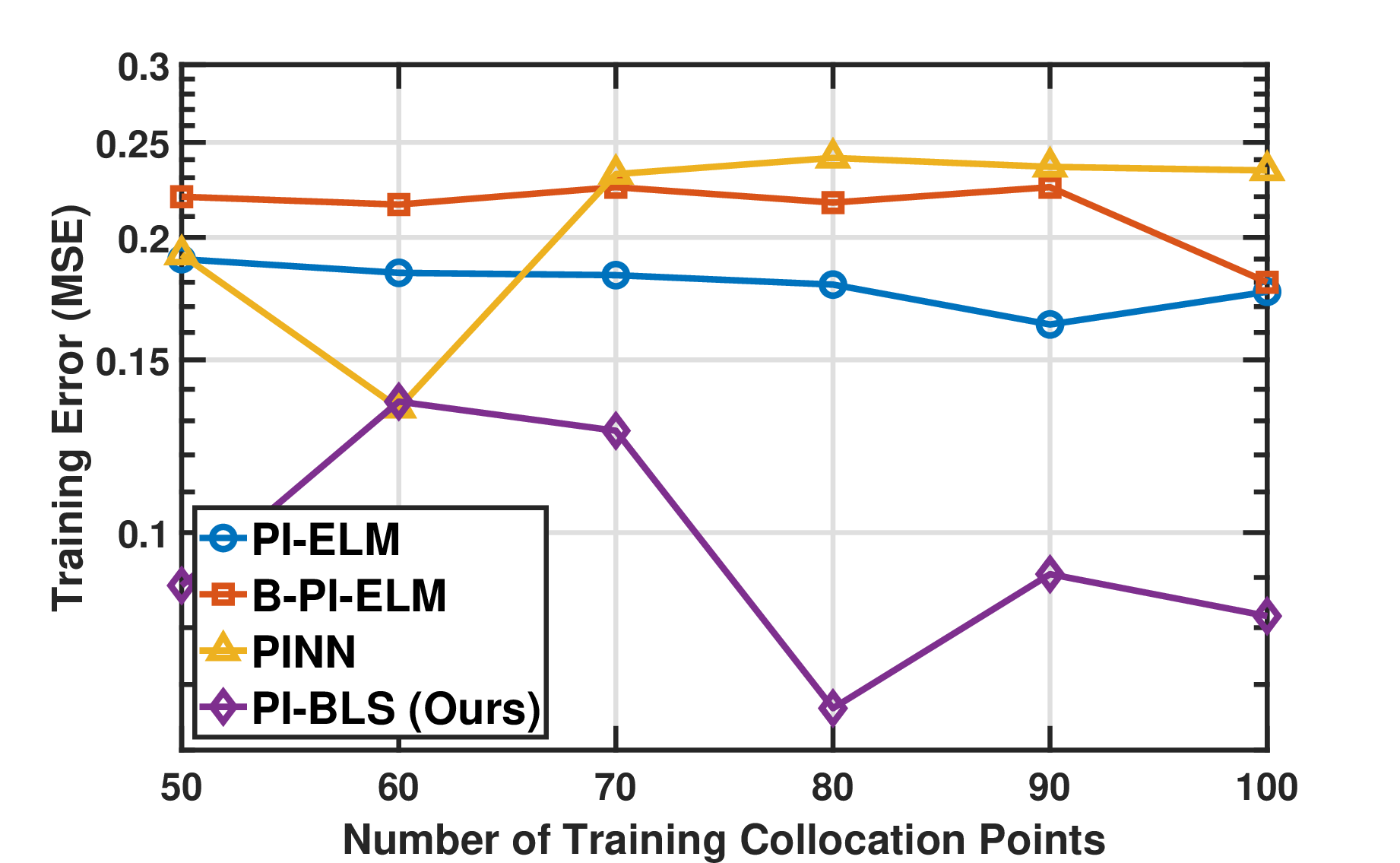}
        \caption*{MSE error in training}
    \end{minipage}
    \begin{minipage}{0.32\textwidth}
        \centering
        \includegraphics[width=\linewidth]{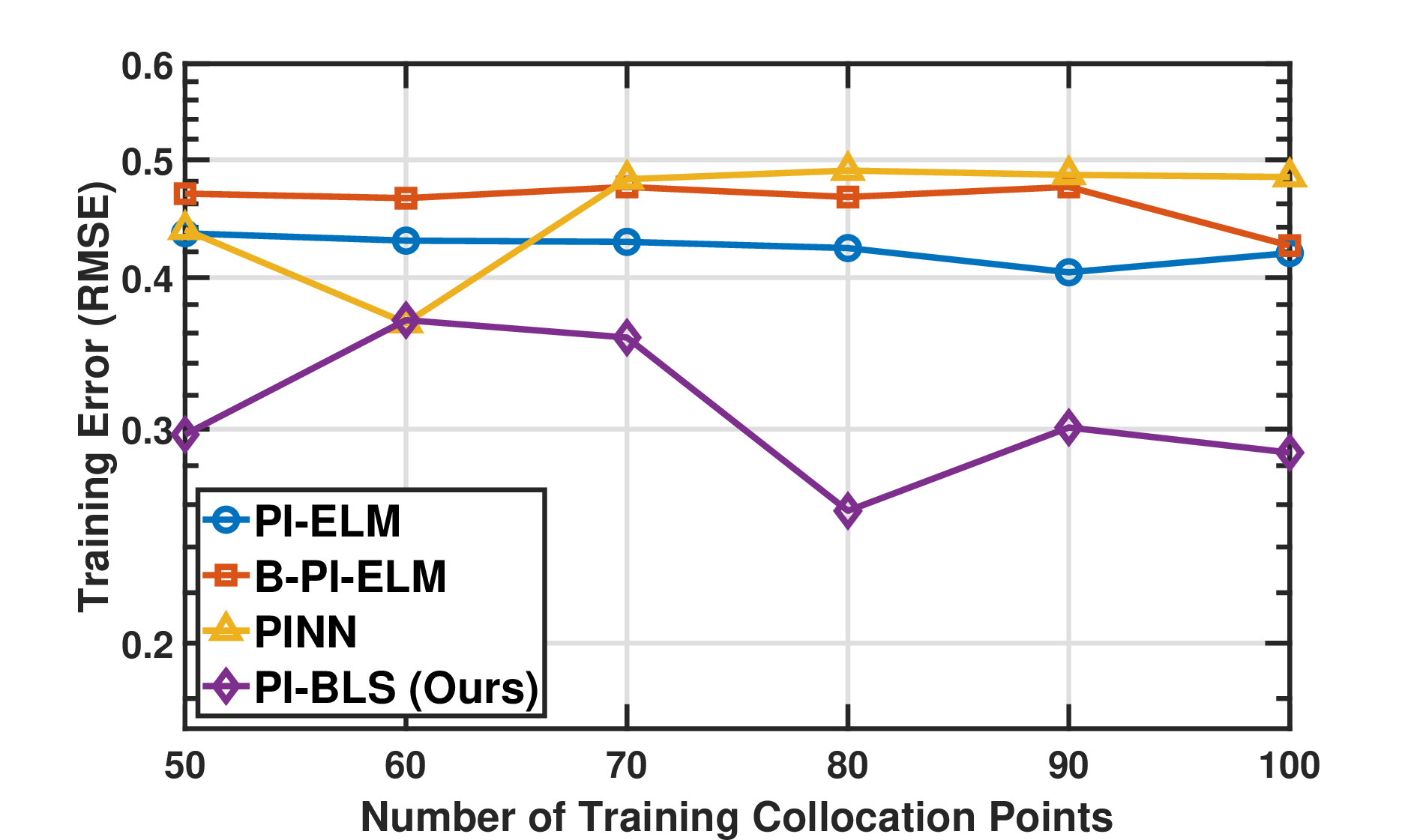}
        \caption*{RMSE error in training}
    \end{minipage}
    \begin{minipage}{0.32\textwidth}
        \centering
        \includegraphics[width=\linewidth]{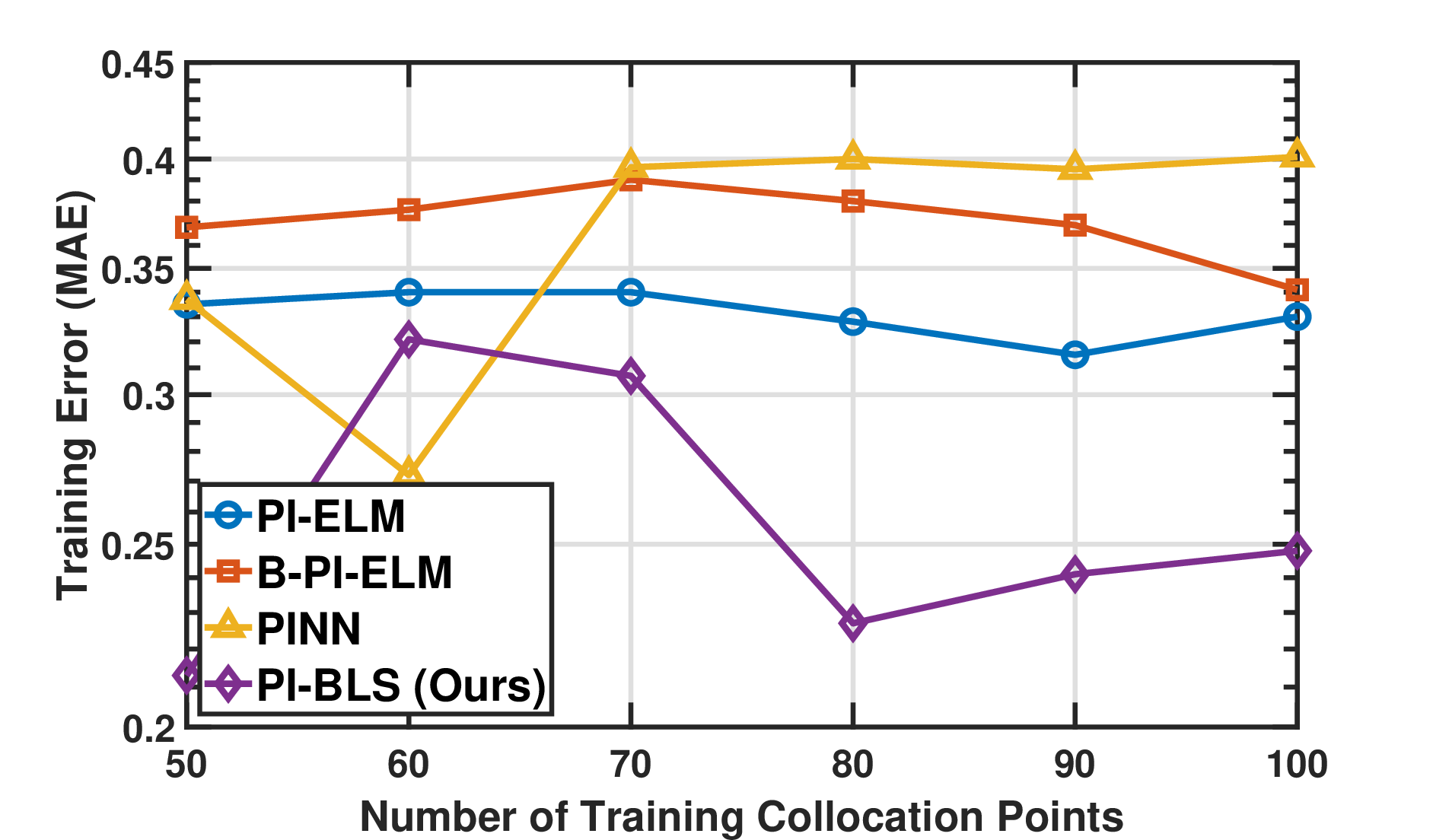}
        \caption*{MAE error in training}
    \end{minipage}

    \vspace{0.5cm} 

    \begin{minipage}{0.32\textwidth}
        \centering
        \includegraphics[width=\linewidth]{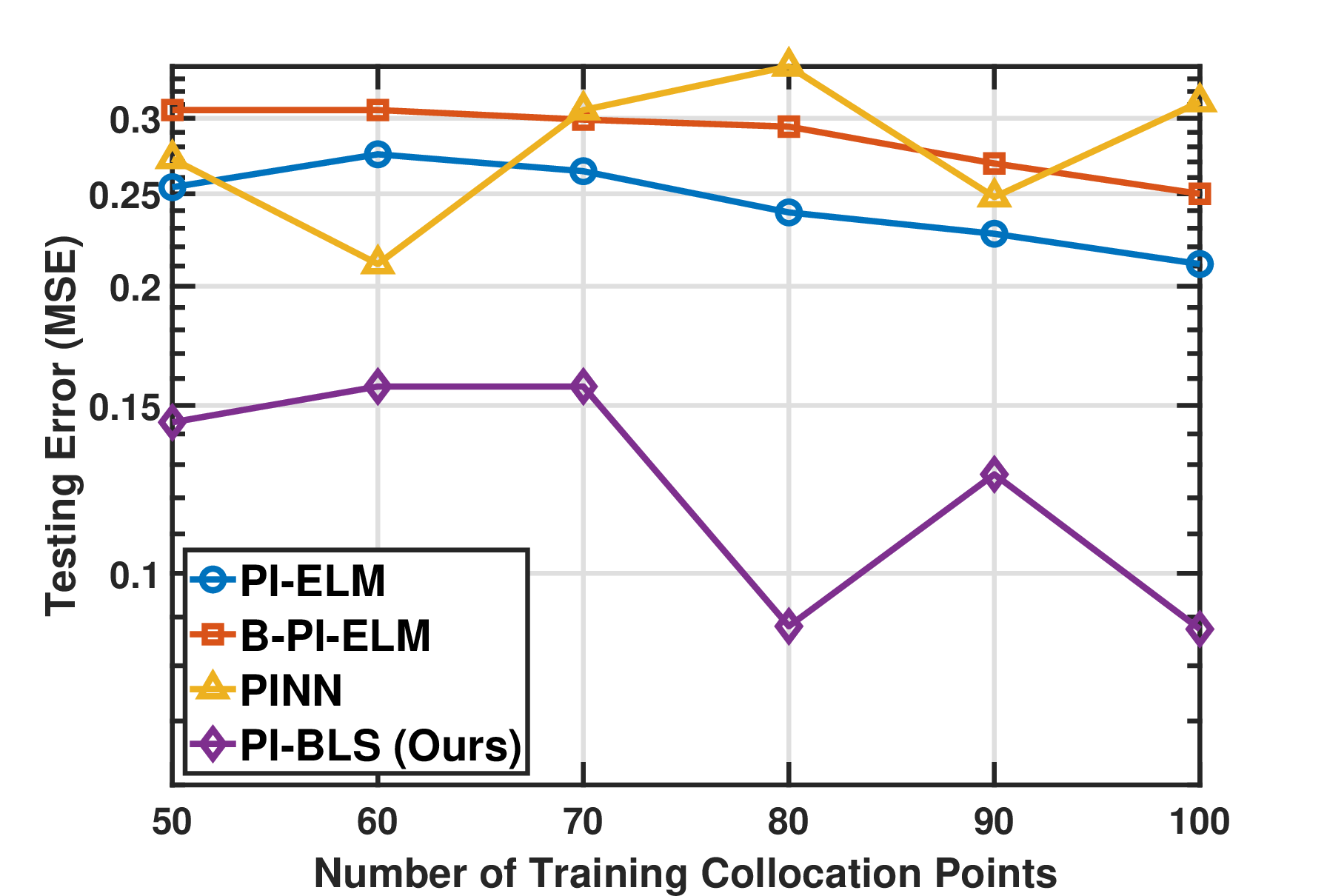}
        \caption*{MSE error in testing}
    \end{minipage}
    \begin{minipage}{0.32\textwidth}
        \centering
        \includegraphics[width=\linewidth]{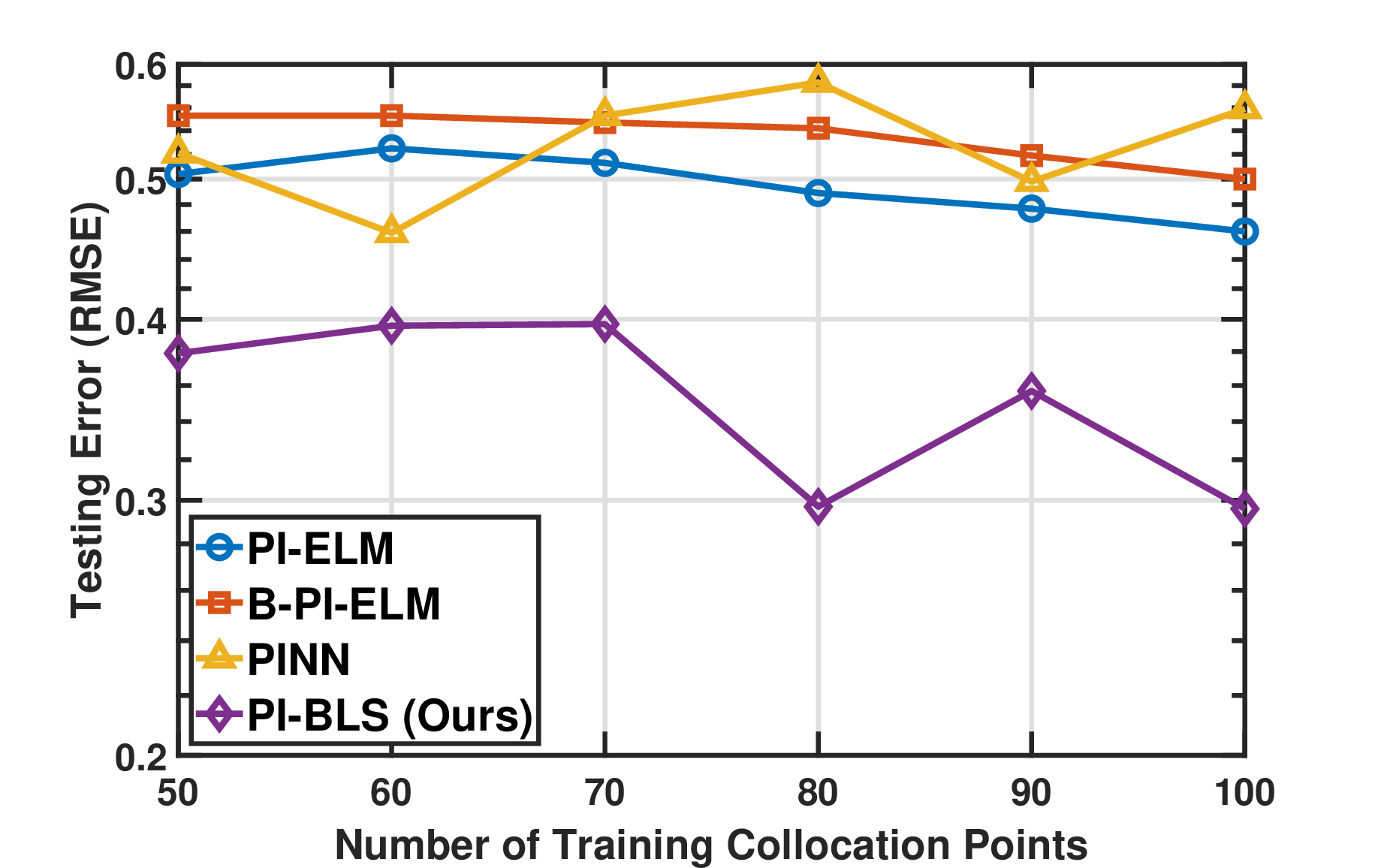}
        \caption*{RMSE error in testing}
    \end{minipage}
    \begin{minipage}{0.32\textwidth}
        \centering
        \includegraphics[width=\linewidth]{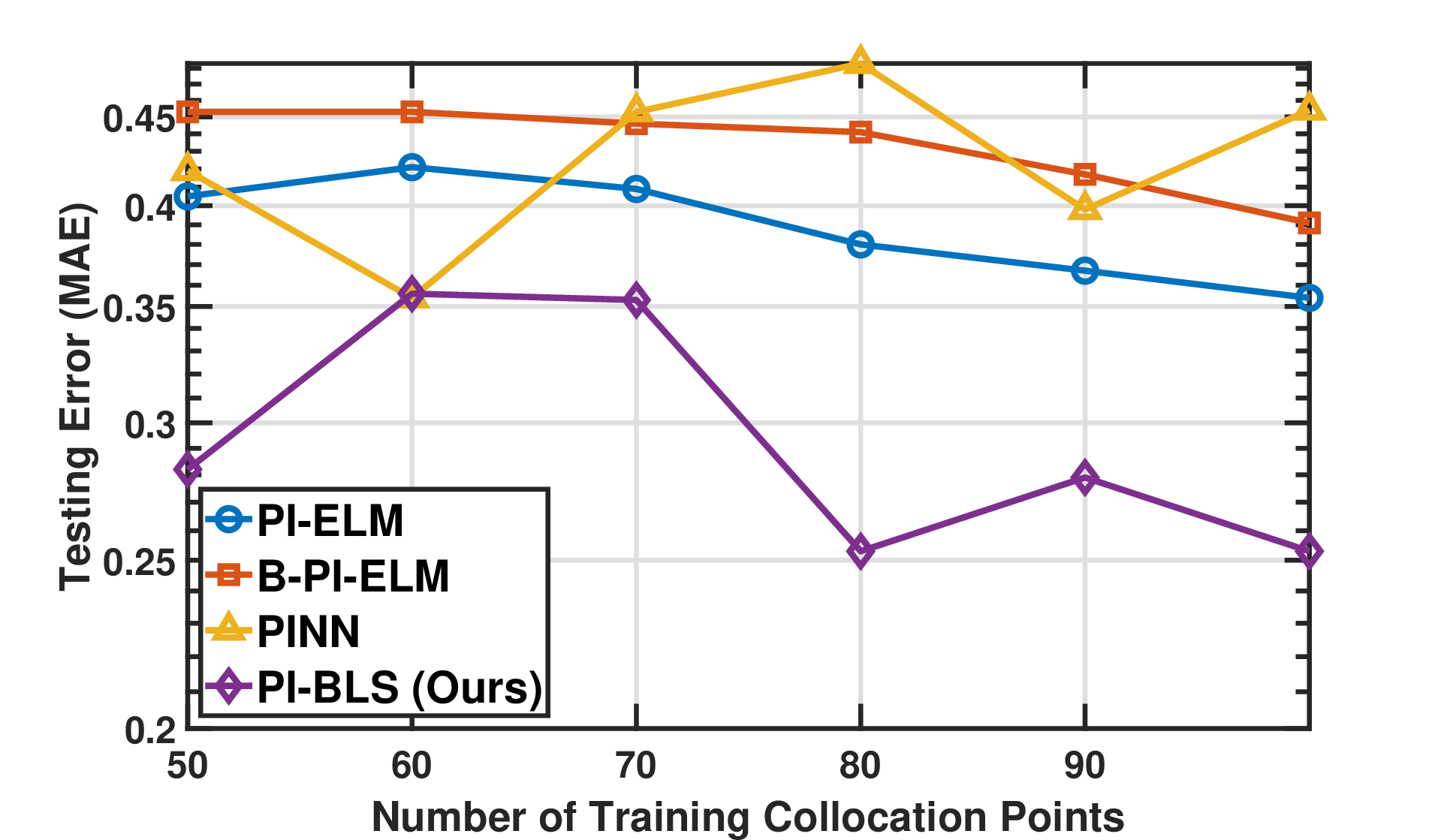}
        \caption*{MAE error in testing}
    \end{minipage}
    \caption{Comparison of various error metrics for the proposed PI-BLS, PINN, PI-ELM, and B-PI-ELM methods across different numbers of training points for the diffusion reaction.}
    \label{fig:comparision_diffusion_reaction}
    \end{figure}
Table~\ref{tab:diffusion_reaction} compares the performance of PI-BLS, PINN, PI-ELM, and B-PI-ELM for the diffusion--reaction equation using 100 training collocation points. Among all methods, PI-BLS consistently achieves the lowest training and testing errors across all evaluation metrics, including MSE, RMSE, MAE, Relative MSE, Relative RMSE, and Relative MAE. The close agreement between the training and testing errors also indicates good generalization capability. In contrast, PINN, PI-ELM, and B-PI-ELM exhibit comparatively larger errors, demonstrating that PI-BLS provides the most accurate and robust approximation for this problem.
    
Figure~\ref{fig:diffusion_reaction_error} compares the true solution, approximate solutions, and absolute residuals obtained by PINN and the proposed PI-BLS method for the diffusion--reaction equation using 50 training collocation points. As shown in Fig.~\ref{fig:diffusion_reaction_error}(a), the PI-BLS solution provides a much more accurate approximation to the true solution, successfully capturing the spatial and temporal variations of the exact profile.  In comparison, the PINN approximation exhibits noticeable deviations from the true solution in several regions of the computational domain. Further, the absolute residual distributions presented in Fig.~\ref{fig:diffusion_reaction_error}(b) highlight the superior performance of PI-BLS, which maintains consistently lower residuals across the computational domain. These results demonstrate that PI-BLS provides a more accurate and stable approximation than PINN, even with a relatively small number of training collocation points.

Figure \ref{fig:comparision_diffusion_reaction} compares the training and testing performance of PI-BLS with PI-ELM, BPI-ELM, and PINN using MSE, RMSE, and MAE across different numbers of training collocation points. The results consistently show that PI-BLS achieves the lowest error values for all three metrics in both the training and testing phases, demonstrating its superior approximation accuracy and better performance in solving the diffusion-reaction equation. In contrast, PI-ELM, BPI-ELM, and PINN exhibit comparatively higher error trends as the number of collocation points increases. These observations confirm that the proposed PI-BLS framework provides more accurate and robust solutions for the diffusion-reaction equation over a wide range of testing data sizes.

\subsection{CPU Time Comparison for Proposed PI-BLS vs PINN}
\begin{table}[ht!]
\centering
\caption{Comparison of training time and time per collocation point for PINN and PI-BLS across different test cases, averaged over 10 runs. The experiments use varying test data sizes, with 30 training collocation points for the advection equation and 500 for both Poisson 2D and diffusion-reaction.}
\label{tab:time_comparison}
\resizebox{\textwidth}{!}{
\begin{tabular}{lcccccc}
\toprule
 & \multicolumn{2}{c}{Advection equation} & \multicolumn{2}{c}{Poisson equation} & \multicolumn{2}{c}{Diffusion-reaction equation} \\
\cmidrule(lr){2-3} \cmidrule(lr){4-5} \cmidrule(lr){6-7}
 & PINN & PI-BLS & PINN & PI-BLS & PINN & PI-BLS \\
\midrule
Training Time (s) & 5.0670 & 3.8257 & 7.7663 & 7.2153 & 6.4458 & 4.9763 \\
Time per Collocation Point (s) & $1.69 \times 10^{-1}$ & $1.28 \times 10^{-1}$ & $1.55 \times 10^{-2}$ & $1.44 \times 10^{-2}$ & $1.29 \times 10^{-2}$ & $9.95 \times 10^{-3}$ \\
\bottomrule
\end{tabular}
}

\end{table}

Table~ \ref{tab:time_comparison} compares the computational efficiency of PINN and the proposed PI-BLS for all three test problems in terms of total training time in seconds (s) and time per collocation point. The results show that PI-BLS consistently outperforms PINN, reducing the total training time from 6.4458~s to 4.9763~s, corresponding to an improvement of approximately 22.8\%. Likewise, the average computation time per collocation point for the diffusion-reaction equation decreases from $1.29 \times 10^{-2}$~s for PINN to $9.95 \times 10^{-3}$~s for PI-BLS, representing a reduction of about 22.9\%. These findings demonstrate that PI-BLS provides a more computationally efficient framework while preserving solution accuracy, making it a promising alternative to PINNs for solving PDEs.

\begin{table*}[!t]
\centering
\caption{Comparison of the number of trainable parameters between PINN and the proposed PI-BLS under the optimal hyperparameter configuration. The reduction factor (Red.) is computed as $\text{PINN}/\text{PI-BLS}$. The ``$\times$'' column represents the parameter reduction factor ($\frac{\text{PINN}}{\text{PI-BLS}}$), i.e., the number of times fewer trainable parameters are required by PI-BLS than by PINN.}
\label{tab:parameter_comparison}
\resizebox{\textwidth}{!}{
\begin{tabular}{c|cccc|c|cccc|c|cccc}
\hline
\multicolumn{5}{c|}{\textbf{Advection Equation}}
&
\multicolumn{5}{c|}{\textbf{Poisson equation}}
&
\multicolumn{5}{c}{\textbf{Diffusion--reaction equation}}\\
\cline{1-15}

$\mathbf{N_{\rm train}}$
&
PINN
&
PI-BLS
&
Red. (\%)
&
$\times$
&
$\mathbf{N_{\rm train}}$
&
PINN
&
PI-BLS
&
Red. (\%)
&
$\times$
&
$\mathbf{N_{\rm train}}$
&
PINN
&
PI-BLS
&
Red. (\%)
&
$\times$
\\
\hline

5  &34405&455&98.68&75.62
&
50 &42225&155&99.63&272.42
&
50 &27385&755&97.24&36.27\\

10 &21165&355&98.32&59.62
&
60 &42225&155&99.63&272.42
&
60 &27385&235&99.14&116.53\\

15 &34405&495&98.56&69.51
&
70 &42225&155&99.63&272.42
&
70 &34405&235&99.32&146.40\\

20 &34405&235&99.32&146.40
&
80 &21165&155&99.27&136.55
&
80 &27385&735&97.32&37.26\\

25 &34405&235&99.32&146.40
&
90 &21165&155&99.27&136.55
&
90 &21165&735&96.53&28.80\\

30 &42022&235&99.44&178.82
&
100&21165&155&99.27&136.55
&
100&27385&235&99.14&116.53\\

\hline
\end{tabular}}
\end{table*}

\subsection{Comparison of Trainable Parameters for Proposed PI-BLS vs PINN}

Table~\ref{tab:parameter_comparison} presents a comparison of the number of trainable parameters between the conventional PINN and the proposed PI-BLS under their respective optimal hyperparameter configurations. It is evident that PI-BLS consistently requires substantially fewer trainable parameters than PINN across all three benchmark problems. Specifically, PI-BLS employs only $155$--$755$ trainable parameters, whereas the corresponding PINN models require between $21,\!165$ and $42,\!225$ trainable parameters.

For the one-dimensional advection equation, PI-BLS reduces the number of trainable parameters by approximately $98.3\%$--$99.4\%$, requiring only $235$--$495$ parameters compared with $21,\!165$--$42,\!022$ parameters for PINN. This corresponds to a reduction factor ranging from approximately $60\times$ to $179\times$. For the two-dimensional Poisson equation, PI-BLS consistently utilizes only $155$ trainable parameters for all training configurations, achieving a parameter reduction of approximately $99.3\%$--$99.6\%$, or equivalently, requiring about $137\times$ to $272\times$ fewer trainable parameters than PINN. Similarly, for the diffusion--reaction equation, PI-BLS requires only $235$--$755$ trainable parameters, yielding a reduction of $96.5\%$ to $99.3\%$ compared to the corresponding PINNs, which translates to approximately $29\times$ to $146\times$ fewer trainable parameters.

The remarkable reduction in trainable parameters originates from the underlying learning mechanism of PI-BLS. Unlike PINNs, where all network weights and biases are optimized through iterative backpropagation, PI-BLS fixes the randomly generated feature and enhancement nodes and optimizes only the output-layer parameters through a linear least-squares formulation. Consequently, the optimization problem is significantly smaller, resulting in a drastic reduction in the number of unknown variables. This compact parameterization not only lowers the optimization complexity but also provides a explanation for the superior computational efficiency and reduced training time of PI-BLS demonstrated in the subsequent experiments.

\section{Conclusion and Future Directions}\label{SEC:Conclusion}

This paper presented a novel PI-BLS for solving forward PDEs by combining the computational efficiency of broad RdNNs with the physics-informed learning paradigm. Unlike conventional PINNs, the proposed framework reformulates the physics-informed learning problem as a linear least-squares optimization task, thereby eliminating the need for iterative gradient-based backpropagation while preserving the governing physical laws through the differential equation and the associated boundary and initial conditions. Furthermore, we established the theoretical foundations of the proposed framework by analyzing its stability, residual error bound, sensitivity to perturbations, and physics consistency, providing rigorous guarantees for the robustness and reliability of PI-BLS.

The effectiveness of the proposed framework was extensively evaluated on three representative benchmark problems, namely the one-dimensional advection equation, the two-dimensional Poisson equation, and the diffusion--reaction equation. Experimental comparisons with PINN, PI-ELM, and B-PI-ELM consistently demonstrated the superiority of PI-BLS in terms of both predictive accuracy and computational efficiency. Across different benchmark problems, PI-BLS achieved lower training and testing errors while requiring substantially fewer trainable parameters than conventional PINNs. In particular, PI-BLS reduced the number of trainable parameters by approximately $96.5\%$--$99.6\%$, corresponding to up to $272\times$ fewer optimization variables. This compact formulation also translated into significant computational gains, where the proposed framework reduced the average training time by approximately $24.6\%$, $46.3\%$, and $24.0\%$ for the advection, Poisson, and diffusion--reaction equations, respectively. Moreover, PI-BLS consistently produced lower approximation errors across all benchmark problems; for example, in the two-dimensional Poisson problem, it reduced the testing RMSE from $1.2486\times10^{-2}$ to $2.4225\times10^{-3}$, representing an improvement of approximately $80.6\%$ over the conventional PINN. Overall, the proposed PI-BLS establishes an efficient and scalable alternative to conventional gradient-based physics-informed learning.

\textbf{Limitations and Future Work:} The proposed PI-BLS framework has some limitations that motivate future research:
\begin{itemize}
    \item The current formulation is designed for forward PDEs with known governing equations. Extending PI-BLS to inverse PDE problems and parameter identification constitutes an important future research direction.
    
    \item The present framework is primarily validated on linear PDEs. Extending PI-BLS to nonlinear PDEs will require solving nonlinear least-squares optimization problems, for which efficient iterative solvers such as the Gauss--Newton or Levenberg--Marquardt methods may be incorporated.
    
    \item The randomized feature generation and dense least-squares formulation may become computationally demanding for large-scale or high-dimensional problems. Future work will investigate adaptive broad architectures, sparse representations, and parallel implementations to further improve scalability and efficiency.
\end{itemize}


\section*{Acknowledgment}
Pinki Khatun acknowledges the financial support provided by the Indian Institute of Technology Gandhinagar (IIT Gandhinagar) provided through the project grant No. IP/$52012$. This work was initiated while Pinki Khatun was a Postdoctoral Fellow at the IIT Gandhinagar, India. The author gratefully acknowledges the support and research environment provided by the institute during that period. M. Sajid acknowledges the Indian Institute of Technology Indore for providing financial support through the Translational Research Fellowship (TRF). Abhinav Jha acknowledges partial support by the IIT Gandhinagar Internal Project: IP/$52016$ and INSPIRE Faculty Fellowship Research Grant: DST/INSPIRE/$04$/$2024$/$000202$.

 \bibliographystyle{abbrvnat}
\bibliography{references.bib}
\end{document}